\documentclass{article}

\usepackage{microtype}
\usepackage{graphicx}
\usepackage{subcaption}
\usepackage{booktabs} 

\usepackage{hyperref}

\usepackage[preprint]{icml2026}

\usepackage{amsmath}
\usepackage{amssymb}
\usepackage{mathtools}
\usepackage{amsthm}

\usepackage[capitalize,noabbrev]{cleveref}

\theoremstyle{plain}
\newtheorem{theorem}{Theorem}[section]
\newtheorem{proposition}[theorem]{Proposition}
\newtheorem{lemma}[theorem]{Lemma}

\theoremstyle{definition}
\newtheorem{definition}[theorem]{Definition}
\newtheorem{assumption}[theorem]{Assumption}
\theoremstyle{remark}

\newcommand{\R}{\mathbb{R}}

\newcommand{\N}{\mathbb{N}}

\newcommand{\unit}{\boldsymbol{1}}

\newcommand{\abs}[1]{\left| #1 \right|}

\newcommand{\vertiii}[1]{{\left\vert\kern-0.25ex\left\vert\kern-0.25ex\left\vert #1 
    \right\vert\kern-0.25ex\right\vert\kern-0.25ex\right\vert}}

\renewcommand{\P}{\mathbb{P}}

\newcommand{\bpi}{\boldsymbol{\pi}}

\usepackage{xcolor}

\newcommand{\beq}{\begin{eqnarray*}}
\newcommand{\eeq}{\end{eqnarray*}}
\newcommand{\beqn}{\begin{eqnarray}}
\newcommand{\eeqn}{\end{eqnarray}}

\usepackage{ifmtarg}
\usepackage{xifthen}%
\newcommand{\ent}[1][]{%
\ifthenelse{\isempty{#1}}{%
\mathrm{H}
}{
\mathrm{H}^{(#1)}
}}

\newcommand{\loch}[1][]{%
\ifthenelse{\isempty{#1}}{%
\mathrm{h}
}{
\mathrm{h}^{(#1)}
}}

\newcommand{\Bcal}{\mathcal{B}}
\newcommand{\Dcal}{\mathcal{D}}

\newcommand{\Ical}{\mathcal{I}}

\newcommand{\Lcal}{\mathcal{L}}

\newcommand{\E}{\mathbb{E}}

\newcommand{\be}{\mathbf{e}}
\newcommand{\bx}{\mathbf{x}}
\newcommand{\x}{\mathbf{x}}

\newcommand{\bv}{\mathbf{v}}

\newcommand{\bh}{\mathbf{h}}

\newcommand{\Scal}{\mathcal{S}}

\newcommand{\Vcal}{\mathcal{V}}

\newcommand{\SM}{\mathrm{Softmax}}

\usepackage[textsize=tiny]{todonotes}

\icmltitlerunning{When are Compositions Learnable from RLVR?}

\begin{document}

\twocolumn[
  \icmltitle{When Is Compositional Reasoning Learnable from Verifiable Rewards?}



  \icmlsetsymbol{equal}{*}

  \begin{icmlauthorlist}
    \icmlauthor{Daniel Barzilai}{equal,wis}
    \icmlauthor{Yotam Wolf}{equal,huji}
    \icmlauthor{Ronen Basri}{wis}
  \end{icmlauthorlist}

  \icmlaffiliation{wis}{Weizmann Institute of Science}
  \icmlaffiliation{huji}{Hebrew University of Jerusalem}

  \icmlcorrespondingauthor{Daniel Barzilai}{daniel.barzilai@weizmann.ac.il}
  \icmlcorrespondingauthor{Yotam Wolf}{yotam.wolf@mail.huji.ac.il}

  \icmlkeywords{RLVR, Optimization, Compositional Problems, Reasoning}

  \vskip 0.3in
]



\printAffiliationsAndNotice{\icmlEqualContribution}

\begin{abstract}
  The emergence of compositional reasoning in large language models through reinforcement learning with verifiable rewards (RLVR) has been a key driver of recent empirical successes. Despite this progress, it remains unclear which compositional problems are learnable in this setting using outcome-level feedback alone. 
  In this work, we theoretically study the learnability of compositional problems in autoregressive models under RLVR training.
  We identify a quantity that we call the \emph{task-advantage ratio}, a joint property of the compositional problem and the base model, that characterizes which tasks and compositions are learnable from outcome-level feedback. 
  On the positive side, using this characterization, we show that compositional problems where correct intermediate steps provide a clear advantage are efficiently learnable with RLVR. 
  We also analyze how such an advantage naturally arises in different problems. On the negative side, when the structural advantage is not present, RLVR may converge to suboptimal compositions. We prove that, in some cases, the quality of the base model determines if such an advantage exists and whether RLVR will converge to a suboptimal solution. 
  We hope our analysis can provide a principled theoretical understanding of when and why RLVR succeeds and when it does not. 
\end{abstract}

\section{Introduction}
Reinforcement Learning with Verifiable Rewards (RLVR) has become a common paradigm for training large language models using outcome-level feedback \citep{jaech2024openai,guo2025deepseek,team2025kimi}. In this setting, a model generates an autoregressive sequence of tokens and receives feedback only through a verifier that evaluates the final output. This paradigm is particularly attractive for reasoning and problem-solving tasks, where supervision of intermediate steps is costly, ambiguous, or unavailable. At the same time, RLVR raises a basic theoretical question: when can outcome-level feedback reliably induce \emph{correct intermediate computation}, and when should it be expected to fail?

A central difficulty of RLVR is that feedback is global, while the model’s behavior is determined by a sequence of intermediate decisions. Multiple chains of thought may lead to a verified outcome, and the verifier does not specify which intermediate choices were responsible for success. As a result, even when a ground-truth chain of thought (CoT) exists and achieves perfect verification accuracy, it is not a priori clear that RLVR should recover it. A model may consistently receive positive reward while systematically relying on incorrect intermediate steps, raising the question of whether outcome-level feedback alone is sufficient to induce correct reasoning.

We study this question in a setting designed to isolate the mechanism by which outcome-based feedback propagates to intermediate decisions. We model reasoning as a sequence of \emph{task selections}: at each step, the model chooses among a finite collection of deterministic tasks, each proposing a next token as a function of the current prefix. 

Our main result is a characterization of the expected step-wise learning signal induced by training on verified rollouts. We show that the direction of the update toward a given task is governed by the \emph{task-advantage ratio}: the ratio between the probability of verification success when that task \emph{is} selected and the probability of success when it \emph{is not} selected. If selecting a task increases the chance of eventual verification success, the expected update reinforces it; if it decreases success probability, the update suppresses it. 

When such an advantage ratio is present for the correct task composition, we prove that RLVR converges to the correct composition in a number of iterations that scales quadratically with the CoT length.
However, without such an advantage, RLVR can converge to suboptimal compositions even in the absence of representational or statistical barriers.

Because the task-advantage ratio is a joint property of the compositional problem and the base model, its necessity reveals the roles of both. First, compositions with an inductive structure—where partially correct CoTs yield a statistical advantage in verification—may be learned efficiently with RLVR, while a lack of such structure can require exponential time. Second, if the base model is poor, correct intermediate steps may not lead to correct final solutions often enough, in which case RLVR may converge to suboptimal solutions.

\paragraph{Contributions.}
Our main contributions are as follows:
\begin{itemize}
    \item In Sec. \ref{sec: setting}, we introduce a theoretical framework for analyzing RLVR in autoregressive compositional settings, modeling reasoning as a sequence of task selections trained using outcome-level feedback.
    \item In \cref{thm:exp_updates_main}, we derive an exact expression for the expected step-wise learning signal induced by RLVR, showing that its direction is governed by the task-advantage ratio.
    \item In \cref{thm: pos_main}, we identify a natural structural condition, phrased in terms of the task-advantage ratio, under which RLVR provably recovers the correct CoT. 
    \item In Sec.~\ref{subsec:inductive_structure}, we demonstrate how the complexity of learning various types of problems depends on the task-advantage ratio. 
    \item In Sec. \ref{sec: neg}, we show how the base model may impact the success of RLVR through the task-advantage ratio. In particular, we demonstrate how RLVR may provably fail to learn even very simple compositions when the base model is poor, providing a principled explanation for the limitations of outcome-based training.
\end{itemize}

Together, these results clarify when RLVR can induce correct intermediate reasoning and when it cannot, highlighting the central roles of the problem structure and the base model.

\section{Related Work}
\begin{figure}[t]
    \centering
    \includegraphics[width=0.9\linewidth]{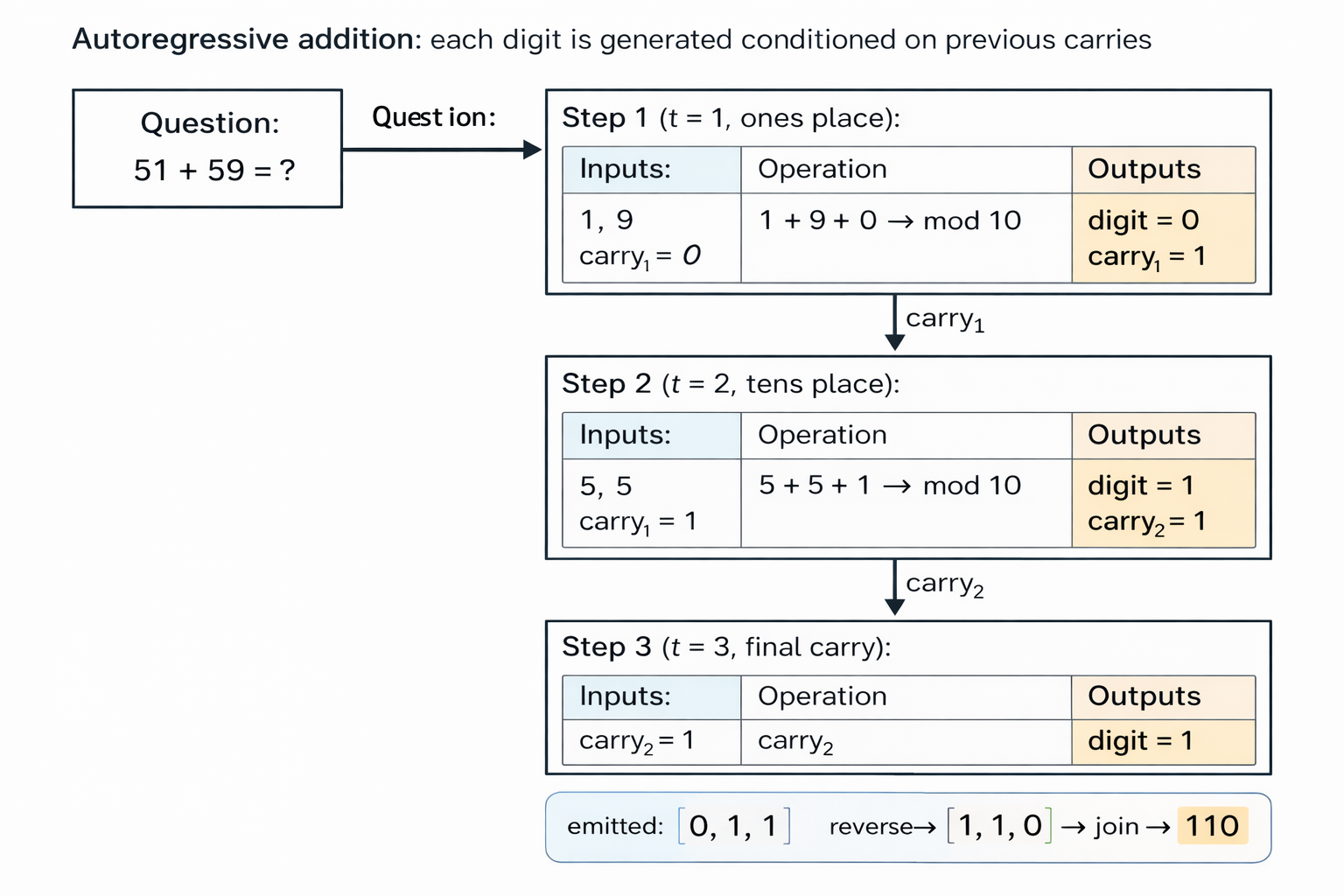}
    \caption{Long addition as an autoregressive compositional problem. LLMs decompose the computation by using their chain of thought as a scratchpad.}
    \label{fig:addition}
    \vspace{-1em}
\end{figure}
\paragraph{RLVR and LLM reasoning. }
RLVR has emerged as a popular paradigm for improving the reasoning capabilities of LLMs, especially on tasks such as math and programming \citep{jaech2024openai, guo2025deepseek, team2025kimi, wang2025reinforcement}. This is largely due to alternative approaches being unstable in many cases \citep{gao2023scaling}. Despite its success, it is still unclear how RLVR changes the reasoning ability of LLMs. On the more critical side, \citet{yue2025does} observed empirically that in many cases, RLVR improved the pass@$k$ accuracy only when $k$ is small; however, for large $k$, the results were often worse than those of the base model. This can be interpreted as RLVR ``sharpening'' the base model, but ultimately relying on knowledge acquired during pre-training. \citet{yuan2025f, wu2025invisible} observe empirically that using RLVR, LLMs may learn to compose skills they learned during pre-training, but the extent to which they can learn completely new skills remains unclear. In our work, we study this compositional ability in depth. On the theoretical side, \citet{foster2025good} showed that in many cases, good coverage of the desired distribution by the base model is necessary for RL methods to succeed. \citet{jia2025we} showed that in some settings, RLVR is no more difficult than when one has intermediate supervision, up to a polynomial in the sequence length.

\paragraph{Learnability of autoregressive compositional problems. }
A core question for reasoning models is how well they can generalize from solving simple tasks to solving compositions of those tasks.
\citet{abedsoltan2025task} analyzed a setting in which tasks are sampled i.i.d. from some distribution of tasks with compositional structure, and infinitely many supervised samples are available to the learner per task. In this setting, they showed that new tasks can readily be identified at inference time. 
\citet{abbe2024far} showed that in some settings, there are problems that can be difficult to learn unless they admit a simple compositional structure. 

Several works consider learning problems where, unlike our setting, intermediate supervision on the CoT is available. We note that in practice, obtaining such supervision is difficult, especially at large scales. \citet{wies2022sub} showed that with supervision on CoTs, an autoregressive model can learn in polynomial time problems that would require exponential time otherwise. \citet{kim2024transformers, wen2024sparse} analyze sparse parities and show that they can be learned efficiently \emph{with intermediate supervision}. Our results complement this by characterizing, through the task-advantage ratio, the difficulty of such problems without intermediate supervision.

\paragraph{Other benefits of Chain of Thought.}
There are many works that analyze theoretically the benefits of the chain of thought from perspectives that differ from ours. One such series of works analyzes how CoTs increase the expressivity of autoregressive models \citep{strobl2024formal}. For example, \citet{merrillexpressive} showed that a transformer empowered with a CoT can simulate a TM of time complexity $T(n)$ with $T(n)$ steps of CoT. Other works analyze how the width and depth of the model interact with the number of CoT steps needed in various types of problems \citep{sanford2024understanding, yehudai2025compositional}. \citet{mirtaheri2025let} show that in some problems long CoTs cannot be efficiently replaced by majority voting between many short ones.

Another line of work analyzes learnability in frameworks related to PAC \citep{malach2023auto, joshi2025theory, shalev2025reasoning, yang2025chain}. For example, \citet{malach2023auto} showed that any Turing machine can be learned by simple autoregressive models in the PAC framework, which is not the case for models trained to output the final answer directly. Of course, the PAC framework has many limitations and cannot accurately capture many aspects of learning.

\section{Learning Compositions Without Supervision}
Many algorithmic problems are naturally expressed as compositions of simpler tasks applied sequentially. In modern language models, such compositions often correspond to chains of thought (CoTs), where intermediate tokens represent partial computations leading to a final answer \citep{wei2022chain, guo2025deepseek}. A central question underlying recent progress in reasoning models is whether such compositional structure can be learned \emph{without} explicit supervision of intermediate steps. To reason about this, we formalize the notion of an autoregressive composition. Throughout this paper, we let $\Vcal$ be a vocabulary and $\Vcal^*$ denote all finite strings over $\Vcal$, meaning $\Vcal^*:=\{(w_1,\ldots, w_n) \mid n\in\N, w_i \in \Vcal\}$ (also known as the Kleene closure of $\Vcal$).

\begin{definition}[Autoregressive Compositions]
    Let $\Vcal$ be a vocabulary and $\mathbb{T}=\{\sigma_j : \Vcal^*\to \Vcal\}_{j=1}^J$ be a set of tasks, where each task maps a sequence of tokens to a single token. A function $f : \Vcal^* \to \Vcal$ is an autoregressive composition (relative to $\mathbb{T}$) if there exist indices $j_1,\ldots,j_S$ such that for any input $\x$, applying $\sigma_{j_1},\ldots,\sigma_{j_S}$ sequentially in an autoregressive manner produces a final token equal to $f(\x)$.
\end{definition}

Autoregressive compositions formalize the idea that a complex computation can be carried out as a sequence of simpler operations, each producing an intermediate token that is then used in subsequent steps. For example, adding two large numbers can be decomposed into many simpler addition tasks (see Fig \ref{fig:addition}).

It is by now known that such CoTs can significantly improve the expressiveness of the model \citep{malach2023auto, merrillexpressive}. The central question we study in this paper is not whether a good composition exists, but whether it can be \emph{learned} when only the correctness of the final output is observed. In particular, we consider a setting in which the individual tasks $\sigma_j$ are already available, but the learner must discover how to compose them.

Learning in this regime is nontrivial because outcome-level feedback does not uniquely identify which intermediate choices were responsible for success. Multiple compositions may yield a verified final output, and incorrect intermediate steps may be compensated for later in the sequence. As a result, even when a correct composition exists and achieves perfect verification accuracy, it is not immediate that training from final outcomes alone should recover a correct chain of intermediate decisions. In this paper, we analyze when outcome-level training induces correct step-wise behavior and when it does not.

\section{Setting and Preliminaries}\label{sec: setting}
\paragraph{Autoregressive generation.} 
We consider an autoregressive language model $f_\theta(\bx)$ that given a context $\bx \in \Vcal^*$ predicts the logits for the next token. This defines a policy for the next token
\[
\bpi_\theta(y \mid \bx) = \SM(f_\theta(\bx))[y],
\]
where $[y]$ denotes the index corresponding to token $y$. 
Given an initial prompt $\bx_0$, the model generates a completion of fixed length $S$ in an autoregressive fashion. Specifically, tokens are generated sequentially according to
\[
    \forall s \in [S], \quad
    y_s \sim \bpi_\theta \left(\cdot \mid \bx_{s-1}\right), \quad
    \bx_s := \left(\bx_{s-1}, y_s\right),
\]
so that $\bx_s$ is the prefix consisting of the prompt followed by the first $s$ generated tokens. This procedure induces a distribution over entire sequences of tokens, which we denote by $\P_\theta$ (so for any token $\bv$, $\P_{\theta}(y_s = v)$ (the probability of sampling $v$ at step $s$) is determined by the distribution of prompts and that $\P_{\theta}(y_s = v \mid \bx_{s-1}) = \bpi_\theta(v\mid \bx_{s-1})$).

\paragraph{Model structure.}
We model the language model as possessing a finite collection of deterministic tasks (or skills) that can be selected at each autoregressive step.  Each task proposes a next token as a deterministic function of the current prefix, and fine-tuning redistributes probability mass over these tasks without altering their internal behavior. This is motivated by empirical work indicating that the core features underlying reasoning behavior are learned during pretraining, and that reinforcement learning primarily selects and sharpens these features \citep{yue2025does,guo2025deepseek}. 

\emph{Tasks.}
Let $\mathbb T = \{\sigma_j : \Vcal^* \to \Vcal\}_{j=1}^J$ denote a set of deterministic tasks, where each task $\sigma_j$ maps a prefix to a single next token, with $\Vcal^*$ denoting the Kleene closure of the vocabulary. We associate to each task a function $g_j : \Vcal^* \to \mathbb R^{|\Vcal|}$ defined by
\[
g_j(\bx) = \gamma \be_{\sigma_j(\bx)},
\]
where $\gamma>0$ is a fixed constant (to be thought of as large) and $\be_v$ denotes the standard basis vector corresponding to token $v$. Intuitively, $g_j(\bx)$ represents the logits proposed by task $\sigma_j$ at prefix $\bx$. When $\gamma$ is large, sampling according to $\SM(g_j(\bx))$ yields the token $\sigma_j(\bx)$ with high probability. Each task may depend arbitrarily on the current prefix, but produces only a single deterministic next token. No single task is assumed to be capable of solving the full problem on its own.

\begin{figure}[t]
    \centering
    \includegraphics[width=1\linewidth]{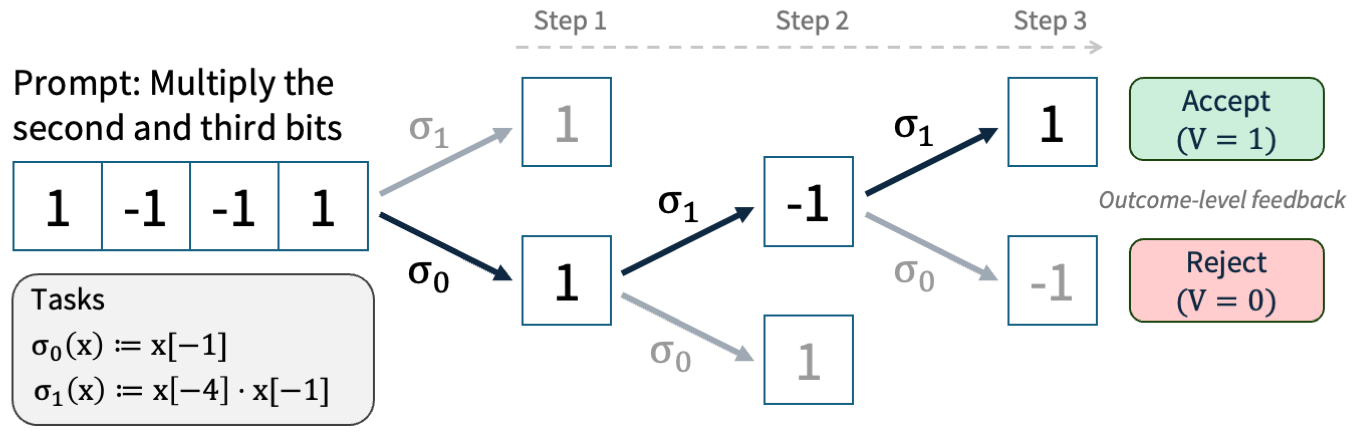}
    \caption{RLVR as sequential task selection. An autoregressive model chooses among tasks at each step, each deterministically producing the next token. Only the final output is verified, and intermediate decisions receive no direct supervision. Dark arrows indicate the correct task composition.}
    \label{fig:parity}
    \vspace{-1em}
\end{figure}

\emph{Positional task embeddings.}
LLMs incorporate positional embeddings that explicitly encode the step, allowing the model to select different tasks at different stages of a chain of thought. To capture this, we associate each task $\sigma_j$ and each autoregressive step $s$ with a positional embedding vector $\bh_{s,j} \in \mathbb R^p$. We assume that the collection $\{\bh_{s,j}\}_{s,j}$ is orthonormal, ensuring that task selection at one step does not interfere with task selection at other steps. This abstraction captures the role of positional embeddings and allows step-specific task selection while keeping tasks disentangled.

\emph{Model definition.} \label{sec: model_def}
Given these components, we define the model as
\begin{align}\label{eq: model}
    f_\theta(\bx) = \sum_{j=1}^J g_j(\bx)\langle \bh_{s(\bx),j}, \theta\rangle,
\end{align}
where $s(\bx)$ denotes the number of generated CoT tokens in $\bx$, and $\theta \in \mathbb R^p$ are the trainable parameters. 

We emphasize that the model above is a behavioral abstraction of an LLM during fine-tuning. Nevertheless, it may also be linked to real LLMs as follows. Many works have shown that a large pretrained model $\tilde f_\theta$ when fine-tuned with a small learning rate satisfies
\begin{align*}
    \tilde f_{\theta}(\bx) \approx \tilde f_{\theta^{(0)}}(\bx) + \nabla_{\theta^{(0)}} \tilde f_{\theta^{(0)}}(\bx)\cdot (\theta - \theta^{(0)}),
\end{align*}
so that the features $\nabla_{\theta^{(0)}}\tilde f_{\theta^{(0)}}(\bx)$ remain fixed during fine-tuning \citep{jacot2018neural, lee2019wide, allen2019convergence, arora2019exact, jang2024lora}.

From this perspective, fine-tuning corresponds to learning a linear combination of fixed pretrained features. One may further decompose these features (e.g., via a singular value decomposition)  $\nabla_{\theta^{(0)}}\tilde f_{\theta^{(0)}}(\x) = \sum_j \tilde g_j(\x) \tilde h_j(\x)^T$, yielding a decomposition of the same for as \eqref{eq: model}.

\paragraph{Compositional target and learning objective.}
We consider problems that admit a correct solution in the form of an autoregressive composition of tasks. For clarity, we restrict attention to compositions of fixed length $S$. Let $\tau(1), \dots, \tau(S) \in [J]$ denote the indices of the tasks appearing in the correct composition. Given an initial prompt $\bx_0$, define the target CoT by composing tasks $\sigma_{\tau(1)}, \ldots, \sigma_{\tau(S)}$. Formally, let $f_0^*(\bx_0):= \bx_0$, and for all $s \in [S]$
\[
f_s^*(\bx_0) := \left(f_{s-1}(\bx_0), \sigma_{\tau(s)}\left(f_{s-1}(\bx_0)\right)\right),
\]
and $f^*(\bx_0):=f_S^*(\bx_0)$ denotes the final output.

\paragraph{Unsupervised post-training.}
We assume access to a verifier $V$ that evaluates the final strings $\bx_{S}$, with $V\left(\bx_{S}\right) \in \{0, 1\}$ denoting an incorrect or correct answer. When the input is clear, we may simply write $V=1$ to denote the event that $V(\bx_S)=1$. We are interested in a procedure where the pre-trained weights $\theta^{(0)}$ are fine-tuned via a procedure analogous to REINFORCE \citep{williams1992simple}. Specifically, at each iteration $t$, we consider a batch of $B$ prompts $\bx^{(t, 1)}_{0}, \ldots, \bx^{(t, B)}_{0}$ sampled i.i.d from some distribution $\Dcal$. We then generate a completion of length $S$ in an auto-regressive fashion, where for each $b\in [B]$ and each step $s\in[S]$ is generated at
\begin{align*}
    y_{s}^{(t, b)} \sim \bpi_t\left(\cdot \mid \bx^{(t, b)}_{s-1}\right) , \quad \bx_{s}^{(t, b)} := \left(\bx^{(t, b)}_{s-1} ~,~ y_{s}^{(t, b)} \right),
\end{align*}
where we use the notation $\bpi_t$ for $\bpi_{\theta^{(t)}}$. We denote by $A_{s,j}$ the event that task $\sigma_j$ is selected at step $s$, equivalently $A_{s,j}=\{y_s=\sigma_j(\bx_{s-1})\}$.

Given a learning rate $\eta > 0$, we update the parameters at each step as
\begin{align*}
    \theta^{(t+1)} = \theta^{(t)} + \frac{\eta}{B} \sum_{b=1}^B V(\bx^{(t,b)}_{S}) \sum_{s=1}^S\nabla \log \bpi_{t}\left(y_{s}^{(t, b)} \mid \bx_{s-1}^{(t, b)}\right).
\end{align*}

For simplicity, we train only on positive samples, i.e., sequences satisfying $V(\bx_{S})=1$. Formally, for each sampled prompt $\bx^{(t,b)}_{0}$, we resample completions until a sequence with $V=1$ is obtained. Also, we do not use any KL regularization. Such regularization limits the model's ability to adapt to new problem structures, as we do here.

We denote by $\P_t$ the distribution over sequences induced by $\theta^{(t)}$, and by $\E_t$ the corresponding expectation. We let $\Bcal_t$ denote the batch of positive samples at time $t$ (sampled from $\P_{\theta^{(t)}}(\cdot \mid V=1)$). 

We emphasize that because the model generates the CoT steps that are fed to the verifier, the dynamics may differ greatly from supervised training. As we will later show explicitly, in this setting, it is possible to converge to poor local minima. 

\paragraph{Assumptions.}
First, since $f^*$ represents the ideal function we are trying to learn we assume that it indeed passes the verifier.
\begin{assumption} \label{ass: value}
    For prompts $\bx_0 \sim \Dcal$ it holds with probability $1$ that $V\left(f^*(\bx_0)\right) = 1$.
\end{assumption}
We emphasize that the target composition need not be the only one that produces a correct final output according to the verifier $V$. Next, we assume that the tasks do not overlap in the following sense:
\begin{assumption} \label{ass: incoherence}
    For any $j\neq j'$ and input $\bx$, it holds that $\sigma_j(\bx)\neq \sigma_{j'}(\bx)$.
\end{assumption}
The assumption is made to simplify analysis. It is always possible to satisfy \cref{ass: incoherence} by expanding the vocabulary. Specifically, one can consider an expanded vocabulary of the form $\Vcal \times [J]$ containing tuples of the form $(v, j)$ which signify token $v$ given by task $\sigma_j$.

\section{Main Results: Learning Compositions With RLVR}\label{sec: positive}
To understand when RLVR can recover an autoregressive composition without intermediate supervision, we analyze the learning signal induced at a single step of the CoT. Since each task deterministically proposes a next token $\sigma_j(\bx_{s-1})$, fine-tuning effectively adjusts the probability with which different tasks are selected at each step. The central question is therefore: \emph{under what conditions does RLVR increase the probability of selecting the correct tasks?}

The answer turns out to depend only on how choosing a given task at step $s$ affects the probability that the final rollout is verified as correct. To formalize this dependence, we introduce the following quantity.

\begin{definition}[task-advantage ratio]\label{def: advantage}
For any parameters $\theta$, step $s$, and task $\sigma_j$, we define the \emph{task-advantage ratio}\footnote{The task-advantage ratio may be $\infty$, in which case we use the convention that $1/\rho^{\theta}_{s,j}=0$. Likewise, if it is $0$ then $1/\rho^{\theta}_{s,j}=\infty$. This is purely for notational convenience and does not impact the results in any way.}
\begin{align*}
    \rho^{\theta}_{s,j}:=
    \frac{\P_{\theta} \left(V=1 \mid A_{s,j}\right)}{\P_{\theta} \left(V=1 \mid A_{s,j}^c\right)}, \quad \text{ and } \quad \rho^{(t)}_{s,j}:= \rho^{\theta^{(t)}}_{s,j} ~~ \forall t,
\end{align*}
where $A_{s,j}$ denotes the event that task $\sigma_j$ is selected at step $s$, and $A_{s,j}^c$ its complement (the event $\sigma_j$ is not selected).
\end{definition}

Intuitively, the task-advantage ratio measures whether choosing a task at a given step makes verified success more or less likely than avoiding it. The following theorem shows that the task-advantage ratio fully determines if RLVR will increase or decrease the probability of selecting each task. Importantly, the expression in the theorem is exact and isolates the effects of outcome-level feedback on selecting specific tasks.

\begin{theorem}\label{thm:exp_updates_main}
    Under assumptions \ref{ass: value}-\ref{ass: incoherence}, for any CoT step $s\in[S]$, time $t \in \N \cup \{0\}$, task $j\in[J]$ and new sample $\bx_{s-1}$ that contains a prompt plus $s-1$ additional CoT steps, then if $\rho_{s, j}^{(t)} > 0$, the expectation of the logit updates (with respect to the training batch $\Bcal_t$) at token $\sigma_j(\bx_{s-1})$ is given by 
    \begin{align*}
        & \E_{\Bcal_t}\Bigl[\left(f_{\theta^{(t+1)}}(\bx_{s-1}) - f_{\theta^{(t)}}(\bx_{s-1})\right)[\sigma_j(\bx_{s-1})]\Bigr] \\ 
        = & \eta \gamma^2 \P_t\left(A_{s,j} \mid V=1\right)\left(1 - \P_t\left(A_{s,j}\right)\right)\left(1 - \frac{1}{\rho_{s, j}^{(t)}}\right).
    \end{align*}
    Otherwise, if $\rho_{s, j}^{(t)} = 0$, then the expected value is equal $- \eta \gamma^2\P_t\left(A_{s,j}\right)$.
\end{theorem}

\cref{thm:exp_updates_main} shows that the expected learning signal at each step decomposes into a few interpretable factors. The first is a directional term $1-1/\rho^{(t)}_{s,j}$, whose sign determines the probability of selecting task $\sigma_j$. In particular, if choosing task $\sigma_j$ increases the probability of verification success, then $\rho^{(t)}_{s,j}>1$ and the expected update favors that task; if it decreases success probability, then $\rho^{(t)}_{s,j}<1$ and the update suppresses it. The other important term in \cref{thm:exp_updates_main} is a ``magnitude'' term $\P_t(A_{s,j}\mid V=1)\bigl(1-\P_t(A_{s,j})\bigr)$, which determines how fast the logits change, and vanishes only when the task is either never or always selected. This characterization makes precise how global outcome-based feedback induces step-wise learning dynamics in autoregressive compositions.

\subsection{Successfully Learning Compositions}

Assumption~\ref{ass: value} ensures that the target function $f^*$ always produces a verified solution, but does not specify anything about other task compositions. In many compositional problems, however, correctness is not arbitrary: selecting the appropriate intermediate operation at each step strictly increases the likelihood that the overall computation succeeds. In such settings, the ground-truth CoT is distinguished by being the one whose intermediate steps are \emph{consistently} aligned with verification success. There may be other CoTs that occasionally lead to correct solutions, but one expects deviations from the ground truth CoT to reduce the chance of reaching a correct solution. The following assumption formalizes this property in terms of the task-advantage ratio introduced in \cref{def: advantage}.

\begin{assumption}[Uniform task-advantage]\label{ass:advantage}
Let $\theta^{(0)}$ be the initial policy parameters, and let $\Theta_+=\{\theta:\forall s\in [S],~\P_\theta(A_{s,\tau(s)})\geq\P_{\theta^{(0)}}(A_{s,\tau(s)})\}$. There exists a constant $\alpha > 0$ such that for any step $s \in [S]$, any task $j \in [J]$, and any parameter vector $\theta\in \Theta_+$
\[
\begin{cases}
\rho^{\theta}_{s,j} < 1, & j \neq \tau(s), \\
\rho^{\theta}_{s,j} \geq 1 + \alpha, & j = \tau(s),
\end{cases}
\]
\end{assumption}
\begin{figure}
    \centering
    \includegraphics[width=\linewidth]{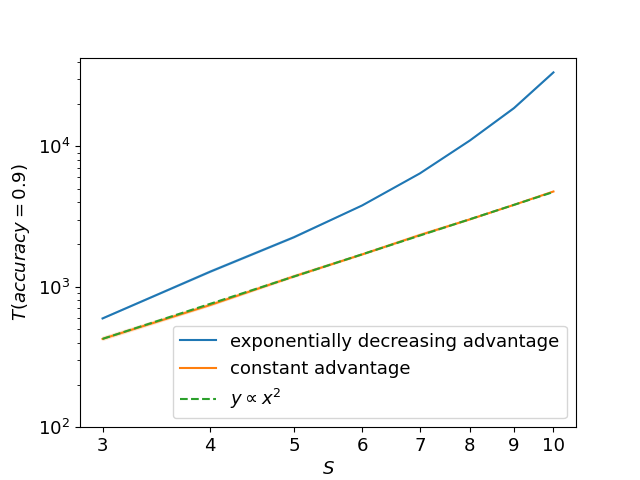}
    \caption{Iteration Complexity needed to reach an accuracy of $0.9$ as a function of the CoT length (log-log scale). Problems with an inductive structure have a constant task advantage, and the iteration complexity scales like $S^2$. Without such a structure (e.g., solving parities), the advantage may decrease exponentially in $S$, and the iteration complexity becomes exponentially large.
    }
    \label{fig:n_vs_T}
    \vspace{-1em}
\end{figure}
Essentially, the assumption implies the correct task has an advantage at each step of the CoT. The parameter $\alpha$ can be thought of as the \emph{margin}, where a larger margin implies a larger advantage that the correct tasks have over incorrect ones. The set $\Theta_+$ limits the policies for which we require the assumption to hold. Instead of requiring an advantage over all possible polices, we require an advantage only within policies that have a higher likelihood of outputting the correct next composition than the initial policy. Indeed, in many problems, if such an advantage exists in the base model, it is maintained throughout training.

\paragraph{Why this represents LLM behavior in reasoning:} In the next section, we will discuss examples of compositional problems satisfying assumption \ref{ass:advantage}, and show that it applies for very general compositional problems. We argue however, that it is particularly consistent with the conditions of typical reasoning problems in LLMs. The reason being that LLMs are known to be strong generators for solutions to such problems, even when exposed to incomplete and incorrect reasoning steps \citep{yu2025formalmath}. 
During RLVR training, correct final answers are generated even if the CoT contains some mistakes; over time, RLVR reinforces reasoning steps that are correlated with the true composition.

\paragraph{Positive Theorem. }
We now show how, under a uniform task advantage, RLVR can indeed converge to a solution that passes the verifier with high probability. The proof is deferred to the appendix.
\begin{theorem}\label{thm: pos_main}
    Under assumptions \ref{ass: value}-\ref{ass: incoherence} and \cref{ass:advantage}, for any $\epsilon \in (0, 1/2), \delta \in (0,1)$, let $p_0:=\min_{s \in [S]} \P_0\left(A_{s,\tau(s)}\right)$, and let the batch size $B$ and learning rate $\eta$ be as specified in Appendix \ref{app: main}.
    There exists and absolute constant $C>0$, such that with probability at least $1-\delta$, at time $T := \left \lceil \frac{C \left(1+\frac{1}{\alpha}\right) S^2 \log\left(S / \min(\epsilon, p_0)\right)}{ \min(\epsilon, p_0)^2 } \right\rceil$, 
    \begin{align*}
        \P_T\left(V(\bx_S)=1\right) \geq 1-\epsilon.
    \end{align*}
\end{theorem}

In words, the proof ensures that when the target task indeed exhibits a positive task-advantage ratio, RLVR recovers it. We highlight a few implications of the result. First, the convergence rate depends on the initial probabilities of selecting the correct tasks, reflecting the empirical observation that skills weakly represented during pre-training are substantially harder to amplify through post-training \citep{yue2025does, wu2025invisible}. Second, the dependence of $T$ on the composition length $S$ is quadratic rather than exponential, indicating that—under a uniform task-advantage condition—RLVR can scale to moderately long chains of thought without incurring an exponential sample complexity penalty. Lastly, regarding the task-advantage ratio, the time to convergence scales with $\alpha$ as $1 + 1/\alpha$. So when $\alpha$ is large, convergence is very quick. By contrast, if $\alpha$ is very close to $0$ (indicating the advantage for the correct task is small) then convergence is slow (c.f. \cref{fig:n_vs_T}).

\section{What compositions are learnable without supervision?}\label{sec: examples}
In this section we dissect the properties of compositional problems that are learnable with a verifiable reward. The learnability is characterized by advantage assumption \ref{ass:advantage}, which is a combined property of the compositional problem's structure and the base model, $\bpi_{\theta_0}$. 
In subsection \ref{subsec:inductive_structure}, we explain through examples the role of the inductive structure of the learnable compositional problems and in subsection \ref{sec: neg}, we demonstrate the role of the base model. 

\subsection{The Inductive Structure of Compositional Problems}\label{subsec:inductive_structure}

We showed above that compositional problems that satisfy Assumption \ref{ass:advantage} are learnable without any intermediate supervision. 
Here, we provide intuition through synthetic problems on the task-advantage ratio and the implied inductive structure. 
We will see that without the structure that rewards partially reasoning paths, a compositional problem may still be learned without supervision, but \textit{may} require an exponential time complexity to learn. On the other hand, in problems where partially correct reasoning paths are incentivized, unsupervised RL can learn in iteration complexity that scales at most quadratically in the composition length. See \cref{fig:n_vs_T} for a comparison of these two problems.

\paragraph{No inductive structure - a case study with parities.} 
Bit parities are a canonical example of problems that are difficult to learn \citep{blum1994weakly, kearns1998efficient, abbe2022non, shoshani2025hardness}, and have therefore gained much attention in the literature as a testing bed to compare various algorithms \citep{barak2022hidden, kim2024transformers, wen2024sparse, kou2024matching, abedsoltan2025task}.
Specifically, given inputs on the cube $\{-1, 1\}^d$, and an \emph{unknown} subset of indices $P\subseteq \{1,\ldots, d\}$, the goal is to learn the function $\phi_P(\bx) := \prod_{i\in P} x_i$. 

There are many ways to view the parities as a compositional task. To best match our setting, we consider the case where the string is ``scanned" from left to right, and each bit is either selected to contribute to the parity or ignored. See Fig. \ref{fig:parity} for a visual depiction. The parity $P$ and its size are unknown. The model, as before, is given by \eqref{eq: model}, and the tasks contain two components: $\sigma_1(\x)=x_{-d}\cdot x_{-1}$ which outputs the product between the last bit in the string and the bit $d$ positions from the end, and $\sigma_0(\x)=  x_{-1}$ which ``does nothing'', and just copies the last bit in the string. We assume that the inputs (prompts) $\bx_0 \in \R^{d+1}$ contain $d$ bits sampled uniformly from the cube $\{-1, 1\}^d$ followed by a last bit which is always $1$ (adding this bias term simplifies the setting). 

In this way, the model essentially goes over the string bit by bit, and decides if each bit either contributes to the parity or is ignored. So, if the desired parity is $\phi_P$, for any $s\in[S]$, letting $\tau(s) = \unit_{s\in P}$, the model can always output the right token by composing the tasks $\sigma_{\tau(1)}, \ldots, \sigma_{\tau(d)}$. The verifier accepts a solution iff the final bit is correct.
We now characterize the task-advantage ratio for such a problem. 
\begin{proposition}\label{prop:parity}
    Under the setting above, suppose that the base model yields a uniform distribution across all tasks (meaning that for any $s$, $\langle \bh_{s,j},\theta_0\rangle$ is a constant across $j$'s). Then assumption \ref{ass:advantage} holds with
    \[
        \alpha = \frac{1}{2^{d-1}}.
    \]
\end{proposition}
The proof can be found in appendix \ref{proof:parity}.
As can be seen, the advantage is exponentially decreasing with $d$. This is not a limitation of the framework, but is an inherent property of learning parities \citep{abbe2022non, shoshani2025hardness}. 
The exponential rate is due to a combination of factors. First, incorrect compositions lead to the correct solution half the time (meaning they are uncorrelated with the correct composition). Moreover, the advantage of choosing a correct hidden bit is manifested only if all other hidden bits are chosen in the rest of the CoT. 

One may wonder if such exponential rates naturally arise in practice. This is unlikely, and when using LLMs to solve reasoning problems, the situation is usually different. Typically, incorrect reasoning steps decrease the likelihood of a correct final solution. Moreover, in many problems, an LLM may increase the chance of a correct final solution even if it takes a few wrong steps along the CoT \citep{yu2025formalmath}. In the next chapter, we analyze a structure of compositions that we believe better reflects the properties of problems where LLMs are used.

\begin{figure}[t]
    \centering
    \includegraphics[width=0.9\linewidth]{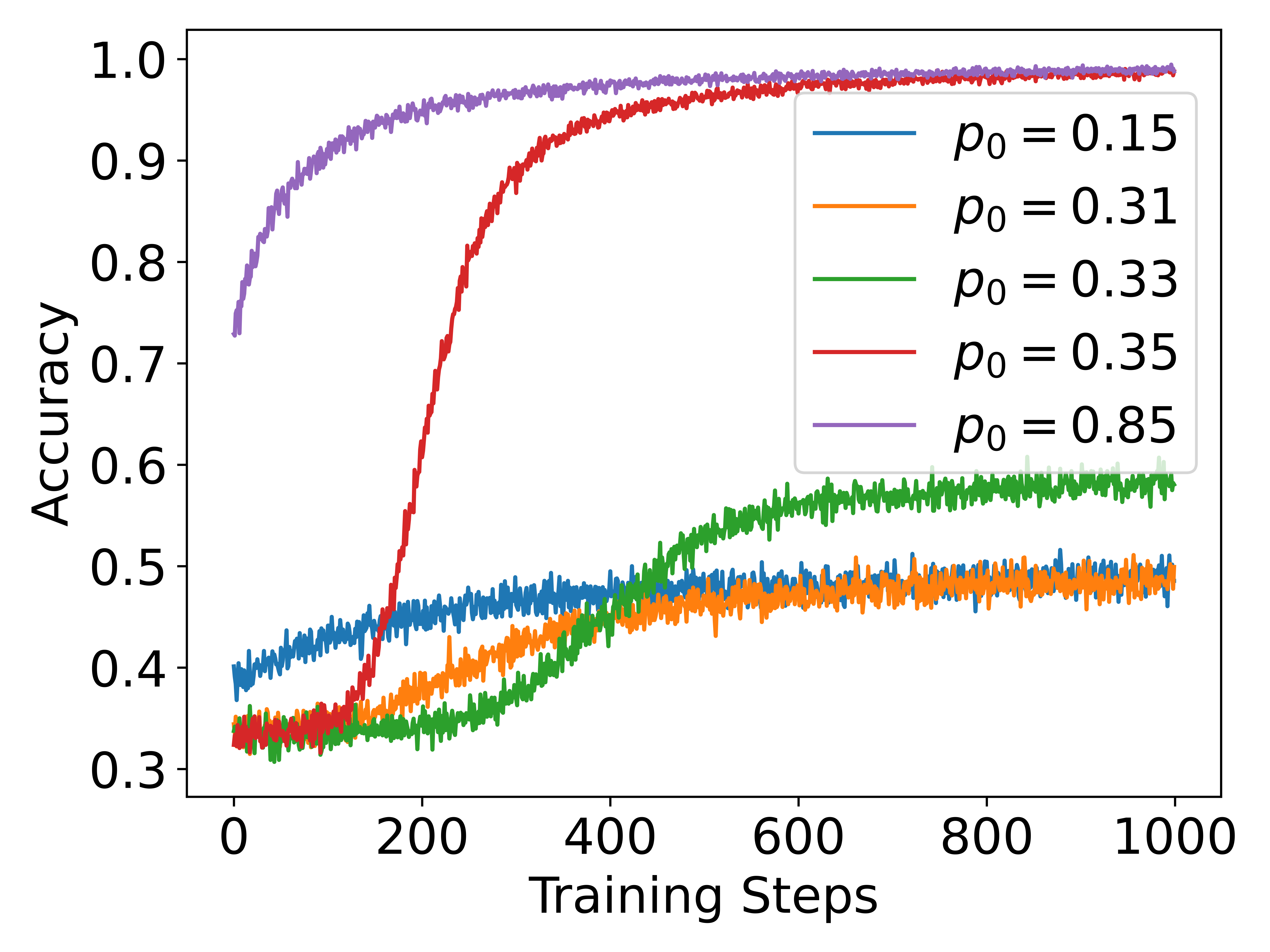}
    \caption{\emph{How the base model can impact RLVR:} We consider a base model whose ``quality'' is determined by $p_0$, the probability of picking the correct tasks at initialization. 
    Within the setting defined in Sec. \ref{sec: neg}, the model converges to a suboptimal solution whenever $p_0<\frac{1}{3}$ (with $\Pr_t(V=1) \to 1/2$) but to the optimal solution whenever $p_0 > 1/3$. 
    This is aligned with the task-advantage ratio $\rho_{s,1}$ which we show is $>1$ when $p_0>1/3$ and $<1$ when $p_0<1/3$. At each step, samples are drawn i.i.d. with a batch size of $256$, and the accuracy is computed with respect to the last batch.
    }
    \vspace{-1em}
    \label{fig:neg}
\end{figure}

\paragraph{Inductive structure - when imperfect task selection is still meaningful.} 
Let us now consider a problem that can be solved by composing the $S$ tasks $\sigma_{\tau(1)},...,\sigma_{\tau(S)}$. 
To better model standard reasoning tasks, we consider a case where each ``correct'' task that is chosen by the policy increases the probability that the verifier accepts. This is meant to model the fact that partially correct CoTs may still help LLMs reach a correct solution \citep{yu2025formalmath}. Note that unlike the parity problem considered earlier, we will show that each correct task chosen provides an advantage.

To formalize this setting, fix some $\lambda_s \in (0, 1)$, and let $\Scal_{\text{bad}}(\bx_S)$ be the set of incorrect tasks chosen by the model. Then the probability the verifier accepts is given by $\prod_{s \in \Scal_{\text{bad}}}\lambda_s$ (where the randomness is over prompts). We also assume that the policy at each step does not depend on the prefix, so that each choice of tasks doesn't influence the probability of choosing a correct task at a future step. Intuitively, $\lambda_s$ should be thought of as the chance of recovering from an incorrect CoT step. 

In this case, not only does assumption \ref{ass:advantage} hold, but it also does not vanish as $S$ is increased:
\begin{proposition}\label{prop:parallelizable}
    Under the above setting,
    assumption \ref{ass:advantage} holds with:
    \[
        \alpha \geq  \min_{s \in [S]}\frac{1 - \lambda_s}{\lambda_s}.
    \]
\end{proposition}
The proof is in appendix \ref{proof:parallelizable}.
We see here that $\alpha$ can remain bounded by a constant as $S$ increases. Consequently, by our main result, RLVR can learn this composition with an iteration complexity that scales quadratically with $S$. This demonstrates the strength of RLVR for reasoning problems, where LLMs can succeed even when initially making incorrect reasoning steps.

RLVR learns this problem quickly because the chance of an incorrect composition to succeed on a given problem drops exponentially with incorrect components. As such, there is a clear advantage to picking the correct task at each intermediate step. This is an important difference from the parity problem, where all incorrect compositions can pass the verifier with the same probability, $1/2$, so acceptance does not meaningfully disambiguate correct intermediate choices.

\subsection{How RLVR Depends on The Base Model}\label{sec: neg}

With a model given by \eqref{eq: model}, one may wonder if RLVR is always expected to succeed. We now complement our result by providing an example in which learning fails because \cref{ass:advantage} does not hold. In fact, surprisingly simple constructions can already provably converge to poor local minima if the base model is poor, and \cref{ass:advantage} is violated. 

To see this, let $\Vcal := \{1, 2 \}$, and the prompt $\bx_0$ is $1$ with probability $1/2$, and $2$ with probability $1/2$. We consider tasks that are simply constant functions, such that $\sigma_1(\bx) := 1$ and $\sigma_2(\bx) := 2$ for all $\bx$ of any length. The verifier is defined as follows:
\begin{align*}
    V(\bx_S) = 
    \begin{cases}
        1 & \bx_S \in \{(1, 1, 1), (2, 1, 1), (1, 2, 2)\} \\ 
        0 & \text{Otherwise}.
    \end{cases}
\end{align*}
As such, the ideal target function $f^*(\bx)$ is given by applying $\sigma_1$ at steps $1$ and $2$ (so $\tau(1) = \tau(2)=1$). 
Even though the model may always achieve a perfect reward if it chooses task $1$, it may also receive a positive reward half the time if it consistently selects the opposite task. We show that the model can easily converge to this local minimum and get stuck.

Consider initial parameters such that for some $p_0\in(0,1)$, for any $\bx_{s-1}$, the probability of picking the correct task is given by $p_0$. The quality of the base model is determined by $p_0$, where a larger $p_0$ indicates a better base model. We now show how $p_0$ impacts the performance of RLVR.

\begin{proposition}\label{prop: simple_neg}
    In the setting above, consider initial parameters such that for some $p_0\in(0,1)$, for any $\bx_{s-1}$, the probability of picking the correct task (task $\sigma_1$) is given by $p_0$. Then for any step size $\eta > 0$,
    \begin{align*}
         \lim_{t \to\infty }\lim_{B\to \infty} \P_{\theta^{(t)}}\left(V(\bx_S) = 1\right)
         = \begin{cases}
             \frac{1}{2} & \text{if } ~~ p_0 < \frac{1}{3} \\
             1 & \text{if } ~~ p_0 > \frac{1}{3}.
         \end{cases}
    \end{align*}

\end{proposition}

The proof is presented in \cref{app: simple_neg}.
Interestingly, in this setting, the success of RLVR depends on the quality of the base model. The proof is based on showing that the task-advantage ratio for the correct task is bigger than $1$ throughout training if the base model picks the correct task with sufficiently large probability, and conversely is smaller than $1$ throughout training if this is not the case. 

This may serve as an explanation for why RLVR is often observed in practice to be highly sensitive to the base model \citep{yue2025does}. If the base model is poor, choosing the correct task at any step may not present a clear advantage in terms of final reward. When the model can receive positive rewards on suboptimal CoTs that lead to the correct solution some of the time, the RLVR dynamics may increase the probability of sampling these suboptimal CoTs instead of sampling the correct one.

\section{Discussion}
We introduced a tractable abstraction for RLVR in which an LLM performs sequential task selection: at each CoT step, the model selects among a finite set of deterministic skills that propose the next token, and RLVR reallocates probability mass across these skills. Within this model, we derived an exact characterization of the step-wise learning signal induced by training on verified rollouts.

Our main conclusion is that RLVR does not inherently recover a ground-truth chain of thought; it reinforces intermediate decisions only insofar as they increase the probability of eventual verification success. This dependence is fully captured by the task-advantage ratio. When correct steps enjoy a uniform advantage (within the relevant policy region), RLVR provably converges to the correct composition with polynomial dependence on chain length; when this structural advantage is absent, RLVR can converge to suboptimal compositions even without representational or statistical limitations, and the outcome can hinge on base-model quality.

A key open direction is to understand how task-advantage manifests in realistic LLMs: can it be measured, predicted from base-model behavior, and controlled via training design (e.g., prompts, verifiers, sampling, or objectives) so as to reliably steer RLVR toward correct intermediate computation?

\section*{Acknowledgements}
Research was supported in part by the Israeli Council for Higher Education (CHE) via the Weizmann Data Science Research Center, by the MBZUAI-WIS Joint Program for Artificial Intelligence Research, and by research grants from the Estates of Tully and Michele Plesser and the Anita James Rosen and Harry Schutzman Foundations, and in part by the ERC (European Research Council) and the ISF (Israel Science Foundation). 


\bibliography{refs}
\bibliographystyle{icml2026}

\newpage
\appendix
\onecolumn

\section{Proofs for Section \ref{sec: positive}} \label{app: positive}
\subsection{Additional Notation}
For any $s\in[S]$, $t \in \N \cup \{0\}$, token $v\in \Vcal$ and new sample $\bx_{s-1}$ that contains a prompt plus $s-1$ additional CoT steps, we define the difference in logits between the ``correct tokens'' $\sigma_{\tau(s)}$ and $v$ as
\begin{align}\label{def: Delta}
        \Delta^{(t)}_{s, v}( \bx_{s-1}) := \left(f_{\theta^{(t+1)}}( \bx_{s-1}) - f_{\theta^{(t)}}( \bx_{s-1})\right)[\sigma_{\tau(s)}( \bx_{s-1})]  - \left(f_{\theta^{(t+1)}}( \bx_{s-1}) - f_{\theta^{(t)}}( \bx_{s-1})\right)[v].
\end{align}
We also define the following quantity for any $j\in[J]$, which captures the empirical discrepancy, over the training batch,
between selecting task $\sigma_j$ at step $s$ and its current policy probability.

\begin{align}\label{def: Q}
    Q^{(t)}_{s, j} := \frac{1}{B}\sum_{b=1}^B \left(\unit_{\{ y_{s}^{(t, b)} = \sigma_j(\bx_{s-1}^{(t, b)}) \}} 
    - \bpi_t\left(\sigma_j(\bx_{s-1}^{(t, b)}) \mid \bx_{s-1}^{(t, b)}\right)\right).
\end{align}

For a policy $\bpi$ we denote by $\bpi\left(\bx_{s-1}\right) \in \R^{\abs{\Vcal}}$ the vector with coordinates given by $\bpi\left(\bx_{s-1}\right)_v = \bpi\left(v \mid \bx_{s-1}\right)$.

Lastly, we denote the cross-entropy loss by
\begin{align*}
    \Lcal(\theta) := \E_{\bx_0 \sim \Dcal}\left[\log \frac{1}{\P_\theta\left( f^*(\bx_0) \mid \bx_0 \right)}\right].
\end{align*}

\subsection{Proof of \cref{thm:exp_updates_main}}
We begin by deriving explicit expressions for the per-step logit updates induced by a single parameter update.

\begin{lemma}\label{lem: logit_updates_raw}
    Under assumptions \ref{ass: value}, \ref{ass: incoherence}, let $ \bx_{s-1}$ be some sample that contains a prompt plus an additional $s-1$ CoT tokens for some $s\in[S]$. Then for any $t\in \N\cup\{0\}$, the logit updates at time $t$ are given by:
    \begin{align}\label{eq: logit_updates}
        f_{\theta^{(t+1)}}( \bx_{s-1}) = & f_{\theta^{(t)}}( \bx_{s-1}) + \frac{\eta}{B} \gamma \sum_{b=1}^B \left(\sum_{j :  y_{s}^{(t, b)} = \sigma_j(\bx_{s-1}^{(t, b)}) } g_j( \bx_{s-1}) - \sum_{j \in [J]} \bpi_t\left(\sigma_j(\bx_{s-1}^{(t, b)}) \mid \bx_{s-1}^{(t, b)}\right) g_j( \bx_{s-1}) \right).
    \end{align}
\end{lemma}
\begin{proof}
    First, the parameter updates for the model (as defined in \eqref{eq: model}) are given by
    \begin{align*}
        \theta^{(t+1)} - \theta^{(t)}= & \frac{\eta}{B} \sum_{b=1}^B \sum_{s'=1}^S\nabla \log \bpi_{t}\left(y_{s'}^{(t, b)} \mid \bx_{s'-1}^{(t, b)}\right) \\ 
        = & \frac{\eta}{B} \sum_{b=1}^B \sum_{s'=1}^S \sum_{j'=1}^J \nabla_\theta f_{\theta^{(t)}}\left( \bx_{s'-1}^{(t, b)} \right)^\top \left(\be_{y_{s'}^{(t, b)}} - \bpi_t \left(\bx_{s'-1}^{(t, b)}\right)\right) \\
        = & \frac{\eta}{B} \sum_{b=1}^B \sum_{s'=1}^S \sum_{j'=1}^J \left\langle g_{j'}\left(\bx_{s'-1}^{(t, b)}\right), \be_{y_{s'}^{(t, b)}} - \bpi_t \left(\bx_{s'-1}^{(t, b)}\right)\right\rangle \bh_{s'-1, j'},
    \end{align*}
     where $\bpi\left(\bx\right) \in \R^{\abs{\Vcal}}$ denotes the vector with coordinates given by $\bpi\left(\bx \right)_v = \bpi\left(v \mid \bx\right)$.
    
    As such, the difference in logits for a new input $\bx_{s-1}$ (with $s-1$ CoT tokens) is given by
    \begin{align*}
        f_{\theta^{(t+1)}}( \bx_{s-1}) = & f_{\theta^{(t)}}( \bx_{s-1}) + \frac{\eta}{B} \sum_{b=1}^B\sum_{s'=1}^S \sum_{j, j' \in [J]} g_j( \bx_{s-1}) \left\langle g_{j'}(\bx_{s'-1}^{(t, b)}), \be_{y_{s'}^{(t, b)}} - \bpi_t (\bx_{s'-1}^{(t, b)}) \right\rangle \left\langle \bh_{s-1, j}, \bh_{s-1, j} \right\rangle.  
    \end{align*}

    Now, using that $\{\bh_{s',j'}\}$ are orthonormal, only $s'=s$ and $j'=j$ survive, and we obtain
    \begin{align*}
        f_{\theta^{(t+1)}}( \bx_{s-1}) = & f_{\theta^{(t)}}( \bx_{s-1}) + \frac{\eta}{B} \sum_{b=1}^B\sum_{j=1}^J g_j( \bx_{s-1}) \left\langle g_j(\bx_{s-1}^{(t, b)}), \be_{y_{s}^{(t, b)}} - \bpi_t (\bx_{s-1}^{(t, b)}) \right\rangle. 
    \end{align*}
    
    Since $g_j(\bx) = \gamma\be_{\sigma_j(\bx)}$, the inner product satisfies
    \begin{align*}
         \left\langle g_j(\bx_{s-1}^{(t, b)}), \be_{y_{s}^{(t, b)}} - \bpi_t (\bx_{s-1}^{(t, b)})\right\rangle 
         = \gamma\left(\delta_{\sigma_j(\bx_{s-1}^{(t, b)}), y_{s}^{(t, b)}} - \bpi_t \left(\sigma_j(\bx_{s-1}^{(t, b)})\mid \bx_{s-1}^{(t, b)}\right)\right)
    \end{align*}
    So
    \begin{align*}
        f_{\theta^{(t+1)}}( \bx_{s-1}) = & f_{\theta^{(t)}}( \bx_{s-1}) + \frac{\eta}{B} \gamma \sum_{b=1}^B \left(\sum_{j :  y_{s}^{(t, b)} = \sigma_j(\bx_{s-1}^{(t, b)}) } g_j( \bx_{s-1}) - \sum_{j \in [J]} \bpi_t\left(\sigma_j(\bx_{s-1}^{(t, b)}) \mid \bx_{s-1}^{(t, b)}\right) g_j( \bx_{s-1}) \right).
    \end{align*}
\end{proof}

\begin{lemma}\label{lem:logit_diff_raw}
    Under assumptions \ref{ass: value}, \ref{ass: incoherence}, let $ \bx_{s-1}$ be some sample that contains a prompt plus an additional $s-1$ CoT tokens for some $s\in[S]$. Then for any $t\in \N\cup\{0\}$ and token $v\in\Vcal$
    \begin{align*}
        \frac{1}{\eta \gamma^2} \left(f_{\theta^{(t+1)}}( \bx_{s-1}) - f_{\theta^{(t)}}( \bx_{s-1})\right)[v] = \sum_{j ~:~  \sigma_{j}(\bx_{s-1}) = v} Q_{s, j}^{(t)},
    \end{align*}
    where $Q_{s, j}^{(t)}$ is defined in \eqref{def: Q}.
\end{lemma}
\begin{proof}
    Using \cref{lem: logit_updates_raw} and that $g_{j}( \bx_{s-1}) = \gamma\be_{\sigma_j(\bx_{s-1})}$, the logit updates for token $v$ are given by
    \begin{align*}
        & \frac{1}{\eta \gamma } \left(f_{\theta^{(t+1)}}( \bx_{s-1}) - f_{\theta^{(t)}}( \bx_{s-1})\right)[v] \\
        = & \frac{1}{B}\sum_{b=1}^B \left(\sum_{j ~:~  y_{s}^{(t, b)} = \sigma_{j}(\bx_{s-1}^{(t, b)}) } g_{j}( \bx_{s-1})[v]  - \sum_{j \in [J]} \bpi_t\left(\sigma_{j}(\bx_{s-1}^{(t, b)}) \mid \bx_{s-1}^{(t, b)}\right) g_{j}( \bx_{s-1})[v]\right) \\ 
        = & \frac{\gamma}{B}\sum_{b=1}^B \sum_{j ~:~  \sigma_{j}(\bx_{s-1}) =v} \left(\unit_{\{ y_{s}^{(t, b)} = \sigma_{j}(\bx_{s-1}^{(t, b)}) \}}  
        - \bpi_t\left(\sigma_{j}(\bx_{s-1}^{(t, b)}) \mid \bx_{s-1}^{(t, b)}\right)\right) \\
        = & \gamma \sum_{j ~:~  \sigma_{j}(\bx_{s-1}) = v} Q_{s, j}^{(t)}.
    \end{align*}
\end{proof}

\begin{lemma}\label{lem:prefix_indif}
    Under assumptions \ref{ass: value}, \ref{ass: incoherence}, for any $j\in[J], s\in[S]$, $t \in \N \cup \{0\}$, and new sample $\bx_{s-1}$ that contains a prompt plus $s-1$ additional CoT steps, it holds that $\bpi_t\left(\sigma_j(\bx_{s-1}) \mid \bx_{s-1}\right) = \P_t\left(A_{s,j}\right)$. 
\end{lemma}
\begin{proof}
    For any $s,j$ let $u^{(t)}_{s-1, j} := \langle \theta^{(t)}, \bh_{s-1, j} \rangle$, such that for prefix $\bx_{s-1}$ (that contains $s-1$ CoT steps),
    \begin{align*}
        f_{\theta^{(t)}}(\bx_{s-1}) = \sum_{j\in [J]} u^{(t)}_{s-1, j} g_j(\bx_{s-1}).
    \end{align*}
    Furthermore, by Assumption \ref{ass: incoherence}, since $\sigma_j(\bx_{s-1})\neq \sigma_{j'}(\bx_{s-1})$ for any $j\neq j'$ and $g_{j'}(\bx_{s-1})= \gamma \be_{\sigma_{j'}(\bx_{s-1})}$, it follows that for any task $\sigma_j$,
    \begin{align*}
        f_{\theta^{(t)}}(\bx_{s-1})[\sigma_j(\bx_{s-1})] 
        = \sum_{j' \in [J]}u^{(t)}_{s-1, j'} \cdot g_{j'}(\bx_{s-1})[\sigma_j(\bx_{s-1})] 
        = \gamma u^{(t)}_{s-1, j}.
    \end{align*}
    Furthermore, let $W_{s-1}\subseteq \Vcal$ denote all tokens $w$ such that $w\neq \sigma_j(\bx_{s-1})$ for all $j\in[J]$. Note that by the above, for all $w
    \in W_{s-1}$, $f_{\theta^{(t)}}(\bx_{s-1})[w] = 0$. As such, summing over all tokens $v \in \Vcal$,
    \begin{align*}
        \sum_{v\in \Vcal} \exp\left(f_{\theta^{(t)}}(\bx_{s-1})[v]\right) 
        = &\sum_{j\in[J]} \exp\left(f_{\theta^{(t)}}(\bx_{s-1})[\sigma_j(\bx_{s-1})]\right) 
        + \sum_{w \in W_{s-1}} \exp\left(f_{\theta^{(t)}}(\bx_{s-1})[w]\right) 
        \\ 
        = & \sum_{j\in [J]} \exp\left(\gamma u^{(t)}_{s-1, j}\right) + \abs{W_{s-1}}.
    \end{align*}
    Notice that this is a constant that depends only on $s-1$ and $t$, but importantly not on $\bx_{s-1}$.
    As such, denoting this constant by $Z_{t, s-1}$, we have
    \begin{align*}
        \bpi_t\left(\sigma_{j}(\bx_{s-1}) \mid \bx_{s-1}\right) = \frac{\exp\left(f_{\theta^{(t)}}(\bx_{s-1})[\sigma_{j}(\bx_{s-1})]\right)}{\sum_{v \in \Vcal} \exp\left(f_{\theta^{(t)}}(\bx_{s-1})[v] \right)} = \frac{\exp(\gamma u^{(t)}_{s-1, j})}{Z_{t, s-1}},
    \end{align*}
    which again, does not depend on $\bx_{s-1}$. By the definition of $\P_t$, 
    \begin{align*}
         \P_t\left(A_{s,j} \mid \bx_{s-1}\right) = \bpi_t\left(\sigma_{j}(\bx_{s-1}) \mid \bx_{s-1}\right),
    \end{align*}
    and the claim follows since this is the same for all $\bx_{s-1}$.
\end{proof}

\begin{lemma}\label{lem:q_value_raw}
    Under assumptions \ref{ass: value}, \ref{ass: incoherence}, for any $j\in[J], s\in[S]$, $t \in \N \cup \{0\}$, and new sample $\bx_{s-1}$ that contains a prompt plus $s-1$ additional CoT steps, it holds that
    \begin{align*}
        \E_t\left[Q^{(t)}_{s, j} \mid V=1\right] = &~ \P_t \left(A_{s,j}\mid V=1\right) - \P_t \left(A_{s,j}\right).
    \end{align*}  
\end{lemma}

\begin{proof}
Using the definition of $Q^{(t)}_{s, j}$, 
\begin{align*}
    \E_t\left[Q^{(t)}_{s, j} \mid V=1\right] = \E_t\left[\unit_{\{y_{s} = \sigma_j(\bx_{s-1})\}} \mid V=1\right] 
    - \E_t\left[\bpi_t\left(\sigma_j(\bx_{s-1}) \mid \bx_{s-1}\right) \mid V=1 \right].
\end{align*}
For the first term, we have by definition that  
\begin{align*}
    \E_t\left[\unit_{\{y_{s} = \sigma_j(\bx_{s-1})\}} \mid V=1\right] = \P_t \left(A_{s,j}\mid V=1\right).
\end{align*}
Since $\bpi_t\left(\sigma_j(\bx_{s-1}) \mid \bx_{s-1}\right)=\P_t\left(A_{s,j} \mid \bx_{s-1}\right)$ by definition, we have that 
\begin{align*}
    \E_t\left[\bpi_t\left(\sigma_j(\bx_{s-1}) \mid \bx_{s-1}\right) \mid V=1 \right] 
    =&_{(*)} \P_t\left(A_{s,j}\right) \E_t\left[1 \mid V=1 \right] \\
    = & \P_t\left(A_{s,j}\right),
\end{align*}
where (*) follows from \cref{lem:prefix_indif}. 
\end{proof}

We now express the expected value of $Q^{(t)}_{s,j}$ in a form that separates the intrinsic advantage of task $\sigma_j$ at step $s$ from the probability of selecting it.

\begin{lemma}\label{lem:q_value}
Under assumptions \ref{ass: value}, \ref{ass: incoherence}, for any $j\in[J], s\in[S]$, $t \in \N \cup \{0\}$, and new sample $\bx_{s-1}$ that contains a prompt plus $s-1$ additional CoT steps, that if $\rho_{s, j}^{(t)} > 0$, then
 \begin{align*}
        \E_t\left[Q^{(t)}_{s, j} \mid V=1\right] 
        = \P_t\left(A_{s,j} \mid V=1\right)\left(1 - \P_t\left(A_{s,j}\right)\right)\left(1 - \frac{1}{\rho_{s, j}^{(t)}}\right). 
    \end{align*}
Otherwise, if $\rho_{s, j}^{(t)} = 0$, then $\E_t\left[Q^{(t)}_{s, j} \mid V=1\right] = -\P_t\left(A_{s,j}\right)$.
\end{lemma}

\begin{proof}
    First, note that for any $j$, By Bayes' law, 
    \begin{align}\label{eq: Bayes}
        \P_t\left(A_{s,j}\right) \P_t\left(V=1 \mid A_{s,j} \right)
        = \P_t\left(A_{s,j} \mid V=1\right)
        \P_t\left(V=1\right). 
    \end{align}
    Note that since softmax assigns non-zero probabilities, it always holds that $\P_t\left(A_{s,j}\right) > 0$ and $\P_t\left(V=1\right) > 0$. Therefore, if $\P_t\left(V=1 \mid A_{s,j}\right)=0$ then also $\P_t\left(A_{s,j} \mid V=1\right)=0$, in which case \cref{lem:q_value_raw} implies
    \begin{align*}
        \E_t\left[Q^{(t)}_{s, j} \mid V=1\right] = - \P_t\left(A_{s,j}\right).
    \end{align*}
    Otherwise, from now on we assume $\P_t\left(A_{s,j} \mid V=1\right)>0$. As such, by \cref{lem:q_value_raw} and \eqref{eq: Bayes}, 
    \begin{align}\label{eq: q_val_helper}
        \E_t\left[Q^{(t)}_{s, j} \mid V=1\right] = &~ \P_t\left(A_{s,j} \mid V=1\right) - \P_t\left(A_{s,j}\right) \nonumber\\ 
        = &~ \P_t\left(A_{s,j} \mid V=1\right)\left(1 - \frac{\P_t\left(V=1\right)}{\P_t\left(V=1 \mid A_{s,j} \right)}\right).
    \end{align}
    Furthermore, by the law of total probability, it holds that
    \begin{align*}
        \P_t\left(V=1\right) 
        = & ~ \P_t\left(V=1 \mid A_{s,j} \right) \P_t\left(A_{s,j}\right) + \\
        & ~~ + \P_t\left(V=1 \mid A_{s,j}^c \right) \P_t\left(A_{s,j}^c\right).
    \end{align*}
    Using the definition of $\rho_{s, j}^{(t)}$, the above equation implies that
    \begin{align*}
        1 - \frac{\P_t\left(V=1\right)}{\P_t\left(V=1 \mid A_{s,j} \right)} 
        = & 1 - \P_t\left(A_{s,j}\right) 
        - \frac{1}{\rho_{s, j}^{(t)}} \left(1 - \P_t\left(A_{s,j}\right)\right) \\ 
        = & \left(1 - \P_t\left(A_{s,j}\right)\right)\left(1 - \frac{1}{\rho_{s, j}^{(t)}}\right).
    \end{align*}
    Plugging this back into \eqref{eq: q_val_helper}, we obtain
    \begin{align*}
        \E_t\left[Q^{(t)}_{s, j} \mid V=1\right] 
        = &~ \P_t\left(A_{s,j} \mid V=1\right)
        \left(1 - \frac{\P_t\left(V=1\right)}{\P_t\left(V=1 \mid A_{s,j} \right)}\right) \\ 
        = &~ \P_t\left(A_{s,j} \mid V=1\right)\left(1 - \P_t\left(A_{s,j}\right)\right)\left(1 - \frac{1}{\rho_{s, j}^{(t)}}\right),
    \end{align*}
    which is exactly what we needed to show.
\end{proof}

The above lemmas immediately imply \cref{thm:exp_updates_main}:
\begin{theorem}[\cref{thm:exp_updates_main} from the main text]\label{thm:exp_updates}
    Under assumptions \ref{ass: value}, \ref{ass: incoherence}, for any CoT step $s\in[S]$, time $t \in \N \cup \{0\}$, task $j\in[J]$ and new sample $\bx_{s-1}$ that contains a prompt plus $s-1$ additional CoT steps, then if $\rho_{s, j}^{(t)} > 0$, the expectation of the logit updates (with respect to the training samples) at token $\sigma_j(\bx_{s-1})$ is given by 
    \begin{align*}
        \E_{\Bcal_t}\Bigl[\left(f_{\theta^{(t+1)}}(\bx_{s-1}) - f_{\theta^{(t)}}(\bx_{s-1})\right)[\sigma_j(\bx_{s-1})]\Bigr] = \eta \gamma^2 \P_t\left(A_{s,j} \mid V=1\right)\left(1 - \P_t\left(A_{s,j}\right)\right)\left(1 - \frac{1}{\rho_{s, j}^{(t)}}\right).
    \end{align*}
    Otherwise, if $\rho_{s, j}^{(t)} =  0$, the expected value is equal to $- \eta \gamma^2\P_t\left(A_{s,j}\right)$.
\end{theorem}
\begin{proof}
    Using \cref{lem:logit_diff_raw} and taking the expectation, we get
    \begin{align*}
         \frac{1}{\eta \gamma^2}\E_{\Bcal_t}\Bigl[\left(f_{\theta^{(t+1)}}(\bx_{s-1}) - f_{\theta^{(t)}}(\bx_{s-1})\right)[\sigma_j(\bx_{s-1})]\Bigr] =  \sum_{j' ~:~  \sigma_{j'}(\bx_{s-1}) = \sigma_{j}(\bx_{s-1})} \E_{\Bcal_t}[Q_{s, j'}^{(t)}] = \E_{\Bcal_t}[Q_{s, j}^{(t)}],
    \end{align*}
    where the last equality uses \cref{ass: incoherence}. Since $\Bcal_t$ are sampled from $\P_t$ conditioned on $V=1$, this is exactly $\E_t\left[Q^{(t)}_{s, j} \mid V=1\right]$. Applying \cref{lem:q_value} completes the proof. 
\end{proof}

\subsection{Additional Lemmas for \cref{thm: pos_main}}
We now lower bound the expected logit gap between the correct token
$\sigma_{\tau(s)}(\bx_{s-1})$ and any incorrect token.
This bound is the key step that drives monotonic improvement of the correct-task probability.
\begin{lemma}\label{lem:delta_diff_exp}
    Under assumptions \ref{ass: value}, \ref{ass: incoherence} and \cref{ass:advantage}, for any $s\in[S]$, $t \in \N \cup \{0\}$, new sample $\bx_{s-1}$ that contains a prompt plus $s-1$ additional CoT steps, and token $v \neq \sigma_{\tau(s)}\left(\bx_{s-1}\right)$, if $\theta^{(t)} \in \Theta_+$ then
    \begin{align*}
        \E_t\left[\frac{1}{\eta \gamma^2} \Delta^{(t)}_{s, v}( \bx_{s-1}) \mid V=1\right] 
        \geq \left(\frac{\alpha}{1+\alpha}\right) \P_t\left(A_{s, \tau(s)} \right)\left(1 - \P_t\left(A_{s, \tau(s)} \right)\right),  
    \end{align*}
    where $\Delta^{(t)}_{s, v}( \bx_{s-1})$ is defined in \eqref{def: Delta}.
\end{lemma}

\begin{proof}
    First, we claim that for any token $v \neq \sigma_{\tau(s)}\left(\bx_{s-1}\right)$, 
    \begin{align}\label{eq: neg_wrong_v}
        \E_{t}\Bigl[\left(f_{\theta^{(t+1)}}(\bx_{s-1}) - f_{\theta^{(t)}}(\bx_{s-1})\right)[v] \mid V=1\Bigr] \leq 0.
    \end{align}
    Indeed, if for all $j\in[J]$, $v \neq \sigma_j(\bx_{s-1})$ then by definition of our model \eqref{eq: model} and the fact that $g_j(\bx) = \gamma\be_{\sigma_j(\bx)}$, it follows that the left hand side of \eqref{eq: neg_wrong_v} equal $0$. Otherwise, if $\exists j\in [J]$ such that $v = \sigma_j(\bx_{s-1})$, then by \cref{thm:exp_updates}, either $\rho_{s, j}^{(t)} = 0$ in which case \eqref{eq: neg_wrong_v} holds, or if $\rho_{s, j}^{(t)} > 0$ then
    \begin{align}
        \E_{t}\Bigl[\left(f_{\theta^{(t+1)}}(\bx_{s-1}) - f_{\theta^{(t)}}(\bx_{s-1})\right)[v] \mid V=1\Bigr] = \eta \gamma^2 \P_t\left(A_{s,j} \mid V=1\right)\left(1 - \P_t\left(A_{s,j}\right)\right)\left(1 - \frac{1}{\rho_{s, j}^{(t)}}\right) \leq 0,
    \end{align}
    where the last inequality follows from \cref{ass:advantage}. So we have shown that \eqref{eq: neg_wrong_v} holds, and thus by the definition of $\Delta^{(t)}_{s, v}$ (\cref{def: Delta}), it follows that 
    \begin{align*}
        \E_t\left[\frac{1}{\eta \gamma^2} \Delta^{(t)}_{s, v}( \bx_{s-1}) \mid V=1\right] \geq \E_{t}\Bigl[\left(f_{\theta^{(t+1)}}(\bx_{s-1}) - f_{\theta^{(t)}}(\bx_{s-1})\right)[\sigma_{\tau(s)}(\bx_{s-1})] \mid V=1\Bigr].
    \end{align*}

    \cref{ass:advantage} ensures that $\rho_{s, \tau(s)}^{(t)} \geq 1 + \alpha > 1$, so by combining this we \cref{thm:exp_updates} we obtain 
    \begin{align}\label{eq: logit_diff_helper2}
        \E_t\left[\frac{1}{\eta \gamma^2} \Delta^{(t)}_{s, v}( \bx_{s-1}) \mid V=1\right] \geq \left(\frac{\alpha}{1+\alpha}\right) \P_t\left(A_{s, \tau(s)}  \mid V=1\right)\left(1 - \P_t\left(A_{s, \tau(s)} \right)\right). 
    \end{align}
    
    Lastly, by \cref{lem:q_value} and \cref{ass:advantage}, $\E_t\left[Q^{(t)}_{s, \tau(s)}\mid V=1\right] \geq 0$. So by \cref{lem:q_value_raw}, it holds that 
    \begin{align*}
        \P_t\left(A_{s, \tau(s)}  \mid V = 1\right) \geq \P_t\left(A_{s, \tau(s)} \right),
    \end{align*}
    As such, we can lower bound \eqref{eq: logit_diff_helper2} by removing the conditioning on $V = 1$, which completes the proof.
\end{proof}

To convert the expected logit improvements into high-probability guarantees,
we next show that all quantities $Q^{(t)}_{s,j}$ (and hence all logit differences)
concentrate uniformly over tasks and CoT steps within a single iteration.
\begin{lemma}\label{lem:diff_hoeffding}
    Under assumptions \ref{ass: value}, \ref{ass: incoherence} and \cref{ass:advantage}, for any $t\in\N \cup \{0\}$ and $\delta>0$, with probability at least $1-\delta$ (over the randomness of the training batch at iteration $t$), it holds simultaneously for all $s\in[S]$, all $j\in[J]$ that
    \begin{align}\label{eq:Q_hoeffding_uniform}
        \big|Q^{(t)}_{s,j} - \E_t[Q^{(t)}_{s,j}\mid V=1]\big|
        \leq \sqrt{\frac{2\log\left(\frac{2SJ}{\delta}\right)}{B}}.
    \end{align}
    Moreover, with the same probability, it holds simultaneously for all $s\in[S]$, all $\bx_{s-1}$ (containing $s-1$ CoT steps), and all tokens $v\in\Vcal$ that
    \begin{align}\label{eq:Delta_hoeffding_uniform}
        \frac{1}{\eta \gamma^2}\Delta^{(t)}_{s,v}(\bx_{s-1})
        \ge
        \E_t\left[\frac{1}{\eta \gamma^2}\Delta^{(t)}_{s,v}(\bx_{s-1})\mid V=1\right]
        - \sqrt{\frac{8\log\left(\frac{2SJ}{\delta}\right)}{B}}.
    \end{align}
\end{lemma}
\begin{proof}
    For fixed $s, j$, recall that
    \[
    Q^{(t)}_{s,j}
    = \frac{1}{B}\sum_{b=1}^B Z^{(t,b)}_{s,j},
    \qquad
    Z^{(t,b)}_{s,j}
    := \unit_{\{ y_{s}^{(t, b)} = \sigma_j(\bx_{s-1}^{(t, b)}) \}}
        - \bpi_t\left(\sigma_j(\bx_{s-1}^{(t, b)}) \mid \bx_{s-1}^{(t, b)}\right).
    \]
    Each $Z^{(t,b)}_{s,j}\in[-1,1]$, and 
    $\{Z^{(t,b)}_{s,j}\}_{b=1}^B$ are i.i.d. across $b$ (where $\bx^{(t,b)}_{s-1}\sim \P_t(\cdot\mid V=1)$). Hence by Hoeffding's inequality, for any $u\geq 0$,
    \[
    \P_t\Big(\big|Q^{(t)}_{s,j} - \E_t[Q^{(t)}_{s,j}\mid V=1]\big|\geq u\Big)
    \leq2\exp\left(-\frac{B u^2}{2}\right).
    \]
    Applying a union bound over all $s\in[S]$ and $j\in[J]$ and choosing $u=\sqrt{\frac{2\log\left(\frac{2SJ}{\delta}\right)}{B}}$ implies \eqref{eq:Q_hoeffding_uniform}.
    
    We now prove \eqref{eq:Delta_hoeffding_uniform}. Fix any $s\in[S]$, prefix $\bx_{s-1}$.
    By \cref{lem:logit_diff_raw}, for any token $w\in\Vcal$,
    \[
    \frac{1}{\eta \gamma^2}\Bigl(f_{\theta^{(t+1)}}(\bx_{s-1})-f_{\theta^{(t)}}(\bx_{s-1})\Bigr)[w]
    = \sum_{j:\sigma_j(\bx_{s-1})=w} Q^{(t)}_{s,j}.
    \]
    By Assumption \ref{ass: incoherence}, for each fixed $\bx_{s-1}$ and token $w$, the set $\{j:\sigma_j(\bx_{s-1})=w\}$ has size at most $1$. Therefore,
    \[
    \frac{1}{\eta \gamma^2}\Delta^{(t)}_{s,v}(\bx_{s-1})
    = \sum_{j:\sigma_j(\bx_{s-1})=\sigma_{\tau(s)}(\bx_{s-1})} Q^{(t)}_{s,j}
    -
    \sum_{j:\sigma_j(\bx_{s-1})=v} Q^{(t)}_{s,j}
    \]
    is a difference of at most two $Q^{(t)}_{s,j}$ terms. Hence, 
    \begin{align*}
    \frac{1}{\eta \gamma^2}\Delta^{(t)}_{s,v}(\bx_{s-1})
    \geq
    \E_t\left[\frac{1}{\eta \gamma^2}\Delta^{(t)}_{s,v}(\bx_{s-1})\mid V=1\right]
    - 2\sup_{s',j'}\big|Q^{(t)}_{s',j'} - \E_t[Q^{(t)}_{s',j'}\mid V=1]\big|,
    \end{align*}
    and plugging in the bound from \eqref{eq:Q_hoeffding_uniform} yields \eqref{eq:Delta_hoeffding_uniform}.
\end{proof}


\begin{lemma}\label{lem: R_t}
    Under assumptions \ref{ass: value}, \ref{ass: incoherence} and \cref{ass:advantage}, for any $s\in[S]$ and $t\in\N\cup\{0\}$, if there exists some $\beta^{(t)}_{s}\in(-1, \infty)$ such that
    \begin{align*}
          \inf_{\bx_{s-1}} \inf_{v\neq \sigma_{\tau(s)}(\bx_{s-1})} \Delta^{(t)}_{s, v}(\bx_{s-1}) \geq \beta^{(t)}_{s},
    \end{align*}
    then for $R^{(t)}_s := \frac{1}{\P_t\left(A_{s, \tau(s)} \right)} - 1$ it holds that
    \begin{align*}
        R^{(t+1)}_s\leq \exp\left(-\beta^{(t)}_s\right) R^{(t)}_s \leq \frac{R^{(t)}_s}{1+ \beta^{(t)}_s}.
    \end{align*}
\end{lemma}
\begin{proof}
    For any $\bx_{s-1}$, it holds that
    \begin{align*}
        \bpi_{t+1}\left(\sigma_{\tau(s)}\left(\bx _{s-1}\right) \mid \bx_{s-1}\right) 
        = & \frac{1}{1 + \sum_{v\neq \sigma_{\tau(s)}\left(\bx _{s-1}\right)} \exp\left(f_{\theta^{(t+1)}}(\bx_{s-1})[v] - f_{\theta^{(t+1)}}(\bx_{s-1})[\sigma_{\tau(s)}\left(\bx _{s-1}\right)]\right)} \\
        = & \frac{1}{1 + \sum_{v\neq \sigma_{\tau(s)}\left(\bx _{s-1}\right)} \exp\left(-\Delta^{(t)}_{s,v}(\bx_{s-1})\right)\exp\left(f_{\theta^{(t)}}(\bx_{s-1})[v] - f_{\theta^{(t)}}(\bx_{s-1})[\sigma_{\tau(s)}\left(\bx _{s-1}\right)]\right)} \\
        \geq & \frac{1}{1 + \exp\left(-\beta_s^{(t)}\right)\sum_{v\neq \sigma_{\tau(s)}\left(\bx _{s-1}\right)} \exp\left(f_{\theta^{(t)}}(\bx_{s-1})[v] - f_{\theta^{(t)}}(\bx_{s-1})[\sigma_{\tau(s)}\left(\bx _{s-1}\right)]\right)}, 
    \end{align*}

    Let $R^{(t)}_s(\bx_{s-1}) := \frac{1}{\bpi_{t}\left(\sigma_{\tau(s)}\left(\bx _{s-1}\right) \mid \bx_{s-1}\right)} - 1$, the above implies
    \begin{align*}
        R^{(t+1)}_s(\bx_{s-1}) 
        \leq & \exp\left(-\beta^{(t)}_s\right) \sum_{v\neq \sigma_{\tau(s)}\left(\bx _{s-1}\right)} \exp\left(f_{\theta^{(t)}}(\bx_{s-1})[v] - f_{\theta^{(t)}}(\bx_{s-1})[\sigma_{\tau(s)}\left(\bx _{s-1}\right)]\right) \\ 
        = & \exp\left(-\beta^{(t)}_s\right) R^{(t)}_s(\bx_{s-1}) \\ 
        \leq & \frac{R^{(t)}_s(\bx_{s-1})}{1+ \beta^{(t)}_s},
    \end{align*}
    where the last inequality uses that $\exp(-x) \leq \frac{1}{1+x}$ for any $x>-1$. 
    Now using \cref{lem:prefix_indif}, it holds for all $\bx_{s-1}$ that $R^{(t)}_s(\bx_{s-1}) = R^{(t)}_s$, which completes the proof.
\end{proof}

\begin{lemma}\label{lem:ver_bound}
    Under Assumptions \ref{ass: value}, \ref{ass: incoherence}, and \cref{ass:advantage}, for any $t\in\N\cup\{0\}$, letting $
    R^{(t)}_s := \frac{1}{\P_t\left(A_{s, \tau(s)} \right)} - 1,
    $
    it holds that
    \[
    \P_t\left(V=1\right) \ge 1 - \sum_{s=1}^S R^{(t)}_s.
    \]
\end{lemma}

\begin{proof}
    By Assumption \ref{ass: value}, if all steps of the correct chain-of-thought are selected then the verifier accepts; that is, the event $
    \bigcap_{s=1}^S A_{s,\tau(s)}$ is contained in the event $V=1$.
    Therefore,
    \begin{align*}
    \P_t(V=1)
    &\ge \P_t\left(\bigcap_{s=1}^S A_{s,\tau(s)}\right)
    = 1 - \P_t\left(\bigcup_{s=1}^S A_{s,\tau(s)}^c\right) \\
    &\ge 1 - \sum_{s=1}^S \P_t\left(A_{s,\tau(s)}^c\right)
    = 1 - \sum_{s=1}^S \left(1 - \P_t\left(A_{s,\tau(s)}\right)\right).
    \end{align*}
    Finally, since $\P_t(A_{s,\tau(s)})\leq1$, we have
    \[
    1 - \P_t(A_{s,\tau(s)}) \leq\frac{1-\P_t(A_{s,\tau(s)})}{\P_t(A_{s,\tau(s)})}
    = \frac{1}{\P_t(A_{s,\tau(s)})}-1
    = R_s^{(t)}.
    \]
    Substituting yields $\P_t(V=1) \ge 1 - \sum_{s=1}^S R^{(t)}_s$.
\end{proof}

\begin{lemma}\label{lem:loss_change}
Under assumptions \ref{ass: value}, \ref{ass: incoherence} and \cref{ass:advantage}, for any $t\in\N\cup\{0\}$, let $\Lcal^{(t)} = \E_{\bx_0}\left[\log\left(\frac{1}{\P_{t}(f^*(\bx_0) \mid \bx_0)}\right)\right]$ denote the cross entropy loss. Then for $R^{(t)}_s := \frac{1}{\P_t\left(A_{s, \tau(s)} \right)} - 1$
it holds that for all $t$, 
\begin{align*}
    \Lcal^{(t)} \leq \sum_{s=1}^S R^{(t)}_s.
\end{align*} 
\end{lemma}

\begin{proof}
    Given a prompt $\bx_0$ let $G_s(\bx_0)$ (or just $G_s$ where $\bx_0$ is clear by context) denote the ``correct'' CoT for the first $s$ steps, given by $(\bx_0, y_1,\ldots, y_{s})$, where for all $s' \leq s$, $y_{s'} = \sigma_{\tau(s')}(\bx_{s'-1})$, and $G_0$ will just be $\bx_0$ (note that by \cref{ass: incoherence}, $G_s(\bx_0)$ is a function of $\bx_0$). Then for any prompt $\bx_0$, by the chain law of probability, 
    \begin{align*}
        \P_t\left(f^*(\bx_0) \mid \bx_0\right) \geq \prod_{s=1}^S \P_t\left(A_{s, \tau(s)}  \mid G_{s-1}\right).
    \end{align*}
    Since $\log\left(1/q\right)$ is decreasing in $q$, we have
    \begin{align*}
        \Lcal^{(t)} = & \E_{\bx_0}\left[\log\left(\frac{1}{\P_{t}(f^*(\bx_0) \mid \bx_0)}\right)\right] \leq \E_{\bx_0}\left[\log\left(\frac{1}{\prod_{s=1}^S \P_t\left(A_{s, \tau(s)}  \mid G_{s-1}\right)}\right)\right] \\ 
        = & \sum_{s=1}^S \E_{\bx_0}\left[\log\left(\frac{1}{\P_t\left(A_{s, \tau(s)}  \mid G_{s-1}\right)}\right)\right] \\ 
        \leq&_{(*)} \sum_{s=1}^S \E_{\bx_0}\left[\frac{1}{\P_t\left(A_{s, \tau(s)}  \mid G_{s-1}\right)} - 1 \right] \\ 
        =&_{(**)} \sum_{s=1}^S \left(\frac{1}{\P_t\left(A_{s, \tau(s)} \right)} - 1\right),
    \end{align*}
    where (*) uses that $\log(\frac{1}{q})\leq \frac{1}{q}-1$ for any $0<q\leq 1$, and (**) follows by \cref{lem:prefix_indif}. So by definition of $R^{(t)}_{s}$, the above equation implies
    \begin{align*}
        \Lcal^{(t)} \leq \sum_{s=1}^S R^{(t)}_{s}.
    \end{align*}
\end{proof}

The following lemma bounds how much a single gradient step can \emph{hurt} the policy. It will be needed to ensure that once the model converges, it will not diverge due to noise arising from the finite batch size. 
\begin{lemma}\label{lem: lipscthiz}
    Under assumptions \ref{ass: value}, \ref{ass: incoherence} and \cref{ass:advantage}, for any $s\in[S]$ and $t\in\N\cup\{0\}$, if $\eta < \frac{1}{4\gamma^2}$
    then
    \begin{align*}
        \P_{t+1}\left(A_{s, \tau(s)} \right) \geq \P_t\left(A_{s, \tau(s)} \right) - 2\eta \gamma^2.
    \end{align*}
\end{lemma}
\begin{proof}
    First, we claim that for any $\bx_{s-1}$ and any $v\neq \sigma_{\tau(s)}(\bx_{s-1})$,
    \begin{align}\label{eq: delta_bound_nJ}
        \abs{\Delta^{(t)}_{s, v}(\bx_{s-1})} \leq 2 \eta \gamma^2.
    \end{align}
    Indeed, by \cref{lem:logit_diff_raw}, for any token $w\in \Vcal$,
    \begin{align*}
        \frac{1}{\eta \gamma^2}\Bigl(f_{\theta^{(t+1)}}(\bx_{s-1})-f_{\theta^{(t)}}(\bx_{s-1})\Bigr)[w]
        = \sum_{j:\sigma_j(\bx_{s-1})=w} Q^{(t)}_{s,j}.
    \end{align*}
    By Assumption \ref{ass: incoherence}, for each fixed $\bx_{s-1}$ and token $w$, the set $\{j:\sigma_j(\bx_{s-1})=w\}$ has size at most $1$. Hence
    \begin{align*}
        \frac{1}{\eta \gamma^2}\Delta^{(t)}_{s,v}(\bx_{s-1})
        = \sum_{j:\sigma_j(\bx_{s-1})=\sigma_{\tau(s)}(\bx_{s-1})} Q^{(t)}_{s,j}
        - \sum_{j:\sigma_j(\bx_{s-1})=v} Q^{(t)}_{s,j},
    \end{align*}
    which is a difference of at most two $Q^{(t)}_{s,j}$ terms. Since each $Q^{(t)}_{s,j}\in[-1,1]$, we obtain \eqref{eq: delta_bound_nJ}.

    Let $\beta^{(t)}_{s} := \inf_{\bx_{s-1}} \inf_{v\neq \sigma_{\tau(s)}(\bx_{s-1})} \Delta^{(t)}_{s, v}(\bx_{s-1})$, then by the choice of $\eta$, we have that $\abs{\beta^{(t)}_{s}} \leq 1/2$. So the conditions of \cref{lem: R_t} are satisfied, and it follows from the definition of $R_s^{(t)}$ and \cref{lem: R_t} that 
    \begin{align*}
        \P_{t+1}\left(A_{s, \tau(s)} \right) - \P_t\left(A_{s, \tau(s)} \right) 
        = &\frac{1}{R_s^{(t+1)} + 1} - \frac{1}{R_s^{(t)} + 1} \\ 
        \geq & \frac{1}{\exp(-\beta^{(t)}_s)R_s^{(t)} + 1} - \frac{1}{R_s^{(t)} + 1} \\
        \geq&_{(*)} -2\eta \gamma^2,
    \end{align*}
    where (*) is because for any $x>0$ and $z\in\R$, $\abs{\frac{1}{1+\exp(-z)x} - \frac{1}{1+x}} \leq |z|$.
\end{proof}

\subsection{Proof of \cref{thm: pos_main}}\label{app: main}
The proof proceeds by tracking the evolution of the error ratios
$R^{(t)}_s = \frac{1}{\P_t(A_{s,\tau(s)})}-1$.
We show that whenever $\P_t(A_{s,\tau(s)})$ is not already close to $1$, 
it increases multiplicatively, and once it is sufficiently large, it cannot decrease by more than a negligible amount. We prove something slightly stronger than the result in the main paper. We also prove that the KL divergence (cross entropy loss) between the target function and the learned model, given by $
    \Lcal(\theta) := \E_{\bx_0 \sim \Dcal}\left[\log \frac{1}{\P_\theta\left( f^*(\bx_0) \mid \bx_0 \right)}\right].
$ converges to $0$.
\begin{theorem}\label{thm: pos_full}
    Under assumptions \ref{ass: value}, \ref{ass: incoherence} and \cref{ass:advantage}, for any $\epsilon > 0, \delta \in (0,1)$, let
    \begin{align} 
        \tilde \epsilon :=& \frac{1}{S}\min\left(\epsilon ~~,~~ \frac{1}{2} ~~,~~ \min_{s \in [S]} \P_0\left(A_{s, \tau(s)}\right)\right) \\ 
        \eta := & \frac{\tilde \epsilon}{8 \gamma^2}, \\
        T := &\left \lceil \frac{96 \left(1 + \frac{1}{\alpha}\right)  \log\left(1 / \tilde \epsilon\right)}{\tilde \epsilon^2 } \right\rceil \\
        B :=&  \frac{128 \left(1+\frac{1}{\alpha}\right)^2  \log\left(\frac{2 J S T}{\delta}\right)}{\tilde \epsilon^2}.
    \end{align}
    Then it holds with probability at least $1-\delta$ that at time $T$, 
    \begin{align*}
        \P_t\left(V=1\right) \geq 1 - \epsilon \qquad \text{and} \qquad \Lcal^{(t)} \leq \epsilon.
    \end{align*}
\end{theorem}
\begin{proof}
    Let $R^{(t)}_s = \frac{1}{\P_t(A_{s,\tau(s)})}-1$. By \cref{lem:ver_bound}, \cref{lem:loss_change}, it suffices to prove that for all $s\in[S]$, $R^{(t)}_s \leq \tilde \epsilon$, since then
    \begin{align*}
        \P_t\left(V=1\right) \geq 1 - \sum_{s=1}^S R^{(t)}_s \geq 1-S\tilde\epsilon \geq 1 - \epsilon \qquad \text{and} \qquad \Lcal^{(t)} \leq \sum_{s=1}^S R^{(t)}_s \leq S\tilde \epsilon \leq \epsilon.
    \end{align*} 
     From Lemmas \ref{lem:delta_diff_exp}, \ref{lem:diff_hoeffding} and the union bound, it holds with probability at least $1-\delta$ that for all $s \in [S]$, $t < T$, $\bx_{s-1}$ and token $v\neq \sigma_{\tau(s)}(\bx_{s-1})$, if $\theta^{(t)} \in \Theta_+$ then
    \begin{align}\label{eq: diff_helper}
        \frac{1}{\eta \gamma^2} \Delta^{(t)}_{s, v}(\bx_{s-1}) 
        \geq & \left(\frac{\alpha}{1+\alpha}\right) \left(1 - \P_t\left(A_{s, \tau(s)} \right)\right) 
        \P_t\left(A_{s, \tau(s)}  \right) - \sqrt{\frac{8 \log\left(\frac{2J S T}{\delta}\right)}{B}} \nonumber\\ 
        \geq&_{(*)}\left(\frac{\alpha}{1+\alpha}\right) \left(1 - \P_t\left(A_{s, \tau(s)} \right)\right) 
        \P_t\left(A_{s, \tau(s)}  \right) - \left(\frac{\alpha}{1+\alpha}\right) \frac{\tilde \epsilon}{4} \nonumber \\
        \geq & \left(\frac{\alpha}{1+\alpha}\right)\left( \left(1 - \P_t\left(A_{s, \tau(s)} \right)\right) 
        \P_t\left(A_{s, \tau(s)}  \right) - \frac{\tilde \epsilon}{4} \right) \nonumber \\
        \geq&_{(**)} \left(\frac{\alpha}{1+\alpha}\right)\left( \left(1 - \P_t\left(A_{s, \tau(s)} \right)\right) 
        \P_t\left(A_{s, \tau(s)}  \right) - \frac{2}{3}\left(1-\frac{\tilde \epsilon }{2} \right)\frac{\tilde \epsilon }{2} \right),
    \end{align}
    where in (*) we used the definition of $B$, and in $(**)$ we used that $\tilde \epsilon \leq 1/2$, so $\frac{\tilde \epsilon}{4} = \frac{2}{3} \frac{\tilde \epsilon }{2} \cdot \frac{3}{4} \leq \frac{2}{3} \frac{\tilde \epsilon }{2} \left(1-\frac{\tilde \epsilon }{2} \right)$. The remainder of the proof assumes that this event indeed occurs.

    Let $\beta^{(t)}_{s} := \inf_{\bx_{s-1}} \inf_{v\neq \sigma_{\tau(s)}(\bx_{s-1})} \Delta^{(t)}_{s, v}(\bx_{s-1})$, $\Ical_1 := [\tilde \epsilon / 2 , 1 - \tilde\epsilon/2]$, $\Ical_2 := [1 - \tilde \epsilon, 1]$ and $T_+ := \inf\{t\geq 0 \mid \theta^{(t)} \notin \Theta_+\}$ (where $\Theta_+$ is defined in \cref{ass:advantage}). Note that by definition of $\Theta_+$, $T_+ \geq 1$.
    The above implies that for all $t<\min(T, T_+)$,
    \begin{align}\label{eq: growth}     
        \textbf{if} \quad \P_t\left(A_{s, \tau(s)} \right) \in \Ical_1 \quad \textbf{then} \quad \beta_s^{(t)} 
        \geq \frac{\eta \gamma^2\left(\frac{\alpha}{1+\alpha}\right)}{3} 
        \left(1 - \P_t\left(A_{s, \tau(s)} \right)\right)
      \P_t\left(A_{s, \tau(s)} \right).
    \end{align}

    So by \cref{lem: R_t}, when the conditions of \eqref{eq: growth} are satisfied, $\P_{t+1}\left(A_{s, \tau(s)} \right) > \P_t\left(A_{s, \tau(s)} \right)$. This implies that for all $t<\min(T, T_+)$,
    \begin{align}\label{eq: growth_weak}     
        \textbf{if} \quad \P_t\left(A_{s, \tau(s)} \right) \in \Ical_1 \quad \textbf{then} \quad P_{t+1}\left(A_{s, \tau(s)}\right) \geq P_{t}\left(A_{s, \tau(s)}\right) \text{ and } P_{t+1}\left(A_{s, \tau(s)}\right)\in \Ical_1 \cup \Ical_2.
    \end{align}
    So we have shown that $\P_t\left(A_{s, \tau(s)} \right)$ increases in $\Ical_1$. We now show that if it reaches $\Ical_2$, it will stay there until time $T$. Specifically, by \cref{lem: lipscthiz} and the definition of $\eta$,  
    \begin{align*}
        \P_{t+1}\left(A_{s, \tau(s)} \right) \geq \P_t\left(A_{s, \tau(s)} \right) - 2\eta \gamma^2 \geq \P_t\left(A_{s, \tau(s)} \right) - \frac{\tilde \epsilon}{4}.
    \end{align*}
    So by this and by \eqref{eq: growth_weak}, it follows inductively that for all $t\in[T]$, $\P_t\left(A_{s, \tau(s)} \right) \geq \P_0\left(A_{s, \tau(s)} \right)$ and $T_+ \geq T$. 
    As such, for all $t \in [T]$
    \begin{align}\label{eq: stays}     
        \textbf{if} \quad \P_t\left(A_{s, \tau(s)} \right) \in \Ical_2\setminus \Ical_1 \quad \textbf{then} \quad P_{t+1}\left(A_{s, \tau(s)}\right) \in \Ical_2.
    \end{align}
    Since $\P_{0}\left(A_{s, \tau(s)}\right) \in \Ical_1 \cup \Ical_2$, \eqref{eq: growth_weak} and \eqref{eq: stays} ensure that $\P_t\left(A_{s, \tau(s)} \right) \in \Ical_1 \cup \Ical_2$ for all $t \in [T]$.
    Furthermore by \eqref{eq: growth_weak} and \eqref{eq: stays}, if there exists a time $\tilde T$ such that $\P_{\tilde T}\left(A_{s, \tau(s)} \right) \geq 1-\tilde \epsilon$, this will remain the case for all $t\in\{\tilde T, \ldots, T\}$. So let,
    \begin{align*}
        \tilde T := \inf\{t\in \N \mid \P_t\left(A_{s, \tau(s)} \right) \in \Ical_2\}, 
    \end{align*}
    then it remains to prove that $\tilde T \leq T$. 

    Since $\P_t\left(A_{s, \tau(s)} \right) \in \Ical_1 \cup \Ical_2$ for all $t\in [T]$, by \cref{lem: R_t} it holds that for any $T \leq \tilde T$
    \begin{align*}
        R_s^{(T)} \leq  \frac{R_s^{(T-1)}}{1 + \beta_s^{(T-1)}} \leq R_s^{(0)} \prod_{j=0}^{T-1}\frac{1}{1 + \beta_{s}^{(j)}} \leq \frac{R_s^{(0)}}{\left(1 + \frac{\eta \gamma^2 \left(\frac{\alpha}{1+\alpha}\right) \tilde \epsilon}{3} \right)^{T}}.
    \end{align*}

    So if $T \geq \frac{\log\left(R^{(0)}_s / \tilde \epsilon\right)}{\log\left(1 + \frac{\eta \gamma^2 \left(\frac{\alpha}{1+\alpha}\right) \tilde \epsilon}{3} \right)}$ then $R_s^{(\min\left(T, \tilde T\right))} \leq \tilde \epsilon$ and thus 
    \begin{align*}
    \P_{\min\left(T, \tilde T\right)}\left(A_{s, \tau(s)}\right) = \frac{1}{1 + R_s^{(\min\left(T, \tilde T\right))}} \geq \frac{1}{1+\tilde \epsilon} \geq 1-\tilde \epsilon,
    \end{align*}
    ensuring $\tilde T \leq T$. So to guarantee that this indeed occurs,  it suffices to take 
    \begin{align*}
        T := &\left \lceil \frac{96 \log\left(1 / \tilde \epsilon\right)}{\left(\frac{\alpha}{1+\alpha}\right) \tilde \epsilon^2 } \right\rceil 
        \geq_{(*)} \left \lceil \frac{6\log\left(R^{(0)}_s / \tilde \epsilon\right)}{\eta \gamma^2 \left(\frac{\alpha}{1+\alpha}\right) \tilde \epsilon}\right \rceil 
        \geq_{(**)} \frac{\log\left(R^{(0)}_s / \tilde \epsilon\right)}{\log\left(1 + \frac{\eta \gamma^2 \frac{\alpha}{1+\alpha} \tilde \epsilon}{3} \right)}.
    \end{align*}
    where (*) follows from the definition of $\eta$ and using that $R_s^{(0)} \leq 1 / \min_{s \in [S]} \P_0\left(A_{s, \tau(s)}\right) \leq 1 / \tilde \epsilon$, and (**) follows from the fact that $\frac{2}{x} \geq \frac{1}{\log\left(1+x\right)}$ for any $x\leq 1$.
\end{proof}

\section{Proof of proposition \ref{prop:parity}}\label{proof:parity}

\begin{proof}
We start by calculating the task-advantage ratios at $\theta^{(0)}$. Let $s\in [d]$ be a step in the CoT. The correct task is $\sigma_{\tau(s)}$ where $\tau(s) := \unit_{s \in P}$. As the initial policy is uniform, if the correct task is chosen, then there is a chance of $(1/2)^{d-1}$ that the other correct tasks will be chosen as well. In that case, the verifier will always return $1$. It is well known that different parities are uncorrelated over the uniform distribution on the cube, meaning that for any other case, the verifier's correctness rate will be $1/2$.
Therefore:
\begin{equation}
    \P (V=1|A_{s,\tau(s)}) = \frac{1}{2}\cdot\left(1-\frac{1}{2^{d-1}}\right) + 1\cdot \frac{1}{2^{d-1}} = \frac{1}{2}\left(1+\frac{1}{2^{d-1}}\right)
\end{equation}
Now, if the correct task is not chosen at step $s$, then for any task selection in the other steps the verifier's correctness rate will be $1/2$, thus:
\begin{equation}
    \P(V=1|A_{s,\tau(s)}^c)=\frac{1}{2}
\end{equation}
Therefore:
\begin{equation}
    \rho_{s,s} = 1+\frac{1}{2^{d-1}}
\end{equation}

Now, let $\sigma_j$ be an incorrect task at step $s$, then again the verifier's correctness rate is $\P(V(\x)=1|A_{s,j})=1/2$. On the other hand, if it is not chosen, then the correct task is chosen, for which we have seen$\P(V(\x)=1|A_{s,j}^c)>1/2$. This means $\rho_{s,j}<1$.

This proves that at initialization, the condition of \ref{ass:advantage} holds.
Now, for any $\theta\in \Theta_+$, consider that the likelihood of selecting the correct task at each step increases. So $\P(V(\x)=1|A_{s,\tau(s)})$ only increases (since there is a higher chance that all the correct tasks will be selected), while $\P(V(\x)=1|A_{s,\tau(s)}^c)$ remains bounded by $1/2$ (as only the fully correct CoT achieves a greater correctness rate).

Thus:
\begin{equation}
    \alpha = \frac{1}{2^{d-1}}
\end{equation}

\end{proof}

\section{Proof of proposition \ref{prop:parallelizable}}\label{proof:parallelizable}

Recall that the probability of the verifier accepting is $\prod_{s\in \Scal_{\text{bad}}}\lambda_s$, and that the policy does not depend on the prefix. 
As such, the ratio between the probability of $V=1$ given $A_{s, \tau(s)}$ (meaning the correct task is chosen at step $s$) and the probability of $V=1$ given $A_{s, \tau(s)}^c$ is $\frac{1}{\lambda_s}$. 
Therefore $\rho_{s,s} = \frac{1}{\lambda_s}$.

If some incorrect task $\sigma_{j}$ is chosen at step $s$, then
\begin{equation}
    \rho_{s,j} = \frac{\lambda_s}{(1-\P_0(A_{s, \tau(s)})) \lambda_s + \P_0(A_{s, \tau(s)})\cdot 1} = \frac{1}{1 - \P_0(A_{s, \tau(s)})\left(1 - \frac{1}{\lambda_s}\right)} < 1
\end{equation}
This yields $\alpha = \min_{s \in [S]}\frac{1}{\lambda_s} - 1$. 

\section{Proof of \cref{prop: simple_neg}}\label{app: simple_neg}
Let $\Vcal = \{1, 2 \}$, and the prompt $\bx_0$ is $1$ with probability $1/2$, and $2$ with probability $1/2$. For all $\bx$ of any length, let $\sigma_1(\bx) = 1$ and $\sigma_2(\bx) =2$. Let $\gamma=1$ so that $g_1(\bx) = [1, 0]^\top$ and $g_2(\bx) = [0, 1]^\top$.
\begin{align*}
    V(\bx_S) = 
    \begin{cases}
        1 & \bx_S \in \{(1, 1, 1), (2, 1, 1), (1, 2, 2)\} \\ 
        0 & \text{Otherwise}.
    \end{cases}
\end{align*}
The target model $f^*$ is given by applying $\sigma_1$ at both steps $1$ and $2$.

Now for any time $t$, step $s$ and task $\sigma_j$ let $z^{(t)}_{s,j} := \langle \theta, \bh_{s,j} \rangle$. Let $\phi(z):=\frac{1}{1+\exp(-z)}$ and for $s\in\{1,2\}$, $Z^{(t)}_s:=z_{s, 1}^{(t)} - z_{s, 2}^{(t)}$. Then for an input $\bx$ that contains a prompt plus an addition $s-1$ CoT tokens (for $s\in \{1,2\}$),
\begin{align}\label{eq: log_example}
    f_{\theta^{(t)}}(\bx_{s-1}) = \sum_{j=1}^2 g_j(\bx_{s-1}) z^{(t)}_{s, j} = 
    \begin{bmatrix}
        z^{(t)}_{s, 1} \\[6pt]
        z^{(t)}_{s, 2}
    \end{bmatrix}.
\end{align}
So at any $s$, the policy is given by $\bpi_t\left(\sigma_1(\bx_{s-1}) \mid \bx_{s-1}\right) = \phi(Z^{(t)}_s)$ for any prompt $\bx_{s-1}$. As such, by the choice of $p_0$, we know that $Z_1^{(0)} = Z_2^{(0)} = \phi^{-1}(p_0)$. Now suppose that there is some time $t$ such that $Z_1^{(t)} = Z_2^{(t)} = \phi^{-1}(p_t)$ (this holds for $t=0$ by assumption). Then, by the problem structure, it is straightforward to compute that for any $s\in\{1,2\}$,
\begin{align*}
    \P_t\left(V=1 \mid A_{s,1}\right) = p_t \qquad \text{and} \qquad
    \P_t\left(V=1 \mid A_{s,2}\right) = \frac{1-p_t}{2}.
\end{align*}
As a result, 
\begin{align}\label{eq: rho_example}
    \rho_{s, 1}^{(t)} = \frac{2p_t}{1-p_t} \qquad \text{and} \qquad \rho_{s, 2}^{(t)} = \frac{1-p_t}{2p_t}.
\end{align}
Furthermore, letting $D(q):= \frac{q^2}{q^2 + \frac{1}{2}(1-q)^2}$, by summing up the probabilities of paths with $V=1$, it is also straightforward to compute that for any $s\in\{1,2\}$,
\begin{align}\label{eq: d_example}
    \P_t\left(A_{s,1} \mid V=1\right) = D(p_t) \qquad \text{and} \qquad \P_t\left(A_{s,2} \mid V=1\right) = 1 - D(p_t).
\end{align}

As such, by \cref{thm:exp_updates}, and by \eqref{eq: log_example} it follows that for any $s,j\in\{1,2\}$
    \begin{align*}
        \E_{\Bcal_t}[z_{s,j}^{(t+1)} - z_{s,j}^{(t)}]
        = \E_{\Bcal_t}\Bigl[\left(f_{\theta^{(t+1)}}(\bx_{s-1}) - f_{\theta^{(t)}}(\bx_{s-1})\right)[\sigma_j(\bx_{s-1})]\Bigr] = \eta \P_t\left(A_{s,j} \mid V=1\right)\left(1 - \P_t\left(A_{s,j}\right)\right)\left(1 - \frac{1}{\rho_{s, j}^{(t)}}\right).
    \end{align*}
So by \eqref{eq: rho_example}, \eqref{eq: d_example}, for $j=1$ it holds for any $s\in\{1,2\}$ that
\begin{align*}
    \E_{\Bcal_t}[z_{s,1}^{(t+1)} - z_{s,1}^{(t)}] = \eta D(p_t)\left(1-p_t\right)\left(\frac{3p_t - 1}{2p_t}\right) = \eta \frac{p_t\left(1-p_t\right)}{2p_t^2 + (1-p_t)^2}(3p_t - 1).
\end{align*}
Likewise, for $j=2$,
\begin{align*}
    \E_{\Bcal_t}[z_{s,2}^{(t+1)} - z_{s,2}^{(t)}] = \eta (1-D(p_t))p_t\left(\frac{1-3p_t}{1-p_t}\right) = \eta \frac{p_t(1-p_t)}{2p_t^2 + (1-p_t)^2} (1 - 3p_t).
\end{align*}
In particular, these imply that
\begin{align}\label{eq: z_updates}
    \lim_{B\to \infty}\left(Z_{s}^{(t+1)} - Z_{s}^{(t)}\right) = 2\eta \frac{p_t\left(1-p_t\right)}{2p_t^2 + (1-p_t)^2}(3p_t - 1)
\end{align}

To ease notation, the remainder of the proof considers the infinite batch size limit.
Since $Z_1^{(0)} = Z_2^{(0)} = \phi^{-1}(p_0)$, it follows by induction that for any $t$, $Z_1^{(t)} = Z_2^{(t)}$. So the conditions for the above derivations holds for any time $t$. We now split the proof into two cases. 

First, if $p_0 > 1/3$, then \eqref{eq: z_updates} implies that $Z_{s}^{(t+1)} \geq Z_{s}^{(t)}$ for all $t$. As such, for any $t, p_t \geq p_0$. So we may lower bound \eqref{eq: z_updates} using $p_0$ as
\begin{align*}
    Z_{s}^{(t+1)} - Z_{s}^{(t)} \geq 2\eta \frac{p_t\left(1-p_t\right)}{2p_t^2 + (1-p_t)^2}(3p_0 - 1) \geq 2 \eta \frac{1-p_t}{2} (3p_0 - 1).
\end{align*}
So suppose by contradiction that $p_t \not \to 1$, then using that $3p_0 - 1>0$, there is some constant $c>0$ depending only on $p_0$ and $\eta$ such that for all $t$, $Z_{t+1} \geq Z_{s}^{(t)} + c$. But then $Z_{s}^{(t)} \to \infty$, implying that $p_t = \frac{1}{1+\exp(-Z_{s}^{(t)})} \to 1$ which is a contradiction. So as $t\to\infty$, $p_t\to 1$ implying that $f_{\theta^{(t)}} \to f^*$.

Now we prove the case of $p_0 < 1/3$. Now, \eqref{eq: z_updates} implies that $Z_{s}^{(t+1)} \leq Z_{s}^{(t)}$ for all $t$. As such, for any $t, p_t \leq p_0$. So we may upper bound \eqref{eq: z_updates} using $p_0$ as
\begin{align*}
    Z_{s}^{(t+1)} - Z_{s}^{(t)} \leq 2\eta \frac{p_t\left(1-p_t\right)}{2p_t^2 + (1-p_t)^2}(3p_0 - 1) \leq 2 \eta p_t (3p_0 - 1).
\end{align*}
So suppose by contradiction that $p_t \not \to 0$, then using that $3p_0 - 1<0$, there is some constant $c>0$ depending only on $p_0$ and $\eta$ such that for all $t$, $Z^{(t+1)}_s \leq Z^{(t)}_s - c$. 
But then $Z^{(t)}_s \to -\infty$, implying that $p_t = \frac{1}{1+\exp(-Z^{(t)}_s)} \to 0$ which is a contradiction. So as $t\to\infty$, $p_t\to 0$ implying that the policy converges to always sampling task $\sigma_2$. In particular this implies the proposition.


\end{document}